\documentclass[twoside,11pt]{article}

\usepackage{blindtext}

\usepackage{jmlr2e}

\usepackage{lastpage}
\usepackage{lastpage}
\jmlrheading{25}{2024}{1-\pageref{LastPage}}{9/23; Revised 9/24}{11/24}{23-1180}{Yifei He, Haoxiang Wang, Bo Li and Han Zhao}

\usepackage{amsmath}
\usepackage[capitalize,noabbrev]{cleveref}

\usepackage{microtype}
\usepackage{graphicx}
\usepackage{dsfont}
\usepackage{booktabs} 
\usepackage{algorithm}
\usepackage{algorithmic}
\usepackage{mathrsfs}
\usepackage{mathtools}
\usepackage{custom_commands}
\usepackage{enumitem}
\usepackage[noorphans,vskip=0ex]{quoting}
\usepackage{xcolor}
\usepackage{xspace}
\usepackage{bm}
\usepackage{bbm}
\usepackage{etoolbox}
\usepackage{url}
\usepackage{lipsum}
\usepackage{hhline}
\usepackage{float}
\usepackage{stfloats}
\usepackage{footmisc}
\usepackage{caption}
\usepackage{mwe}
\usepackage{subcaption}
\usepackage{thmtools}
\usepackage{thm-restate}

\declaretheorem[name=Observation]{obs}

\newcommand{\zerodisplayskips}{%
  \setlength{\abovedisplayskip}{4pt}%
  \setlength{\belowdisplayskip}{4pt}%
  \setlength{\abovedisplayshortskip}{1pt}%
  \setlength{\belowdisplayshortskip}{1pt}}
\appto{\normalsize}{\zerodisplayskips}
\appto{\small}{\zerodisplayskips}
\appto{\footnotesize}{\zerodisplayskips}

\ShortHeadings{Gradual Domain Adaptation: Theory and Algorithms}{He, Wang, Li and Zhao}
\firstpageno{1}

\begin{document}

\title{Gradual Domain Adaptation: Theory and Algorithms}

\author{\name Yifei He$^\ast$ \email yifeihe3@illinois.edu \\
      \addr University of Illinois Urbana-Champaign
      \AND
      \name Haoxiang Wang$^\ast$ \email hwang264@illinois.edu \\
      \addr University of Illinois Urbana-Champaign
      \AND
      \name Bo Li \email bol@uchicago.edu \\
      \addr University of Chicago
      \AND
      \name Han Zhao \email hanzhao@illinois.edu \\
      \addr University of Illinois Urbana-Champaign
      }

\newcommand{\customfootnotetext}[2]{{%
		\renewcommand{\thefootnote}{#1}%
		\footnotetext[0]{#2}}}%
\customfootnotetext{$\ast$}{Equal contribution.}

\editor{Quentin Berthet}

\maketitle

\begin{abstract}
Unsupervised domain adaptation (UDA) adapts a model from a labeled source domain to an unlabeled target domain in a one-off way. Though widely applied, UDA faces a great challenge whenever the distribution shift between the source and the target is large. Gradual domain adaptation (GDA) mitigates this limitation by using intermediate domains to gradually adapt from the source to the target domain. In this work, we first theoretically analyze gradual self-training, a popular GDA algorithm, and provide a significantly improved generalization bound compared with~\cite{kumar2020understanding}. Our theoretical analysis leads to an interesting insight: to minimize the generalization error on the target domain, the sequence of intermediate domains should be placed uniformly along the Wasserstein geodesic between the source and target domains. The insight is particularly useful under the situation where intermediate domains are missing or scarce, which is often the case in real-world applications. Based on the insight, we propose \textbf{G}enerative Gradual D\textbf{O}main \textbf{A}daptation with Optimal \textbf{T}ransport (GOAT), an algorithmic framework that can generate intermediate domains in a data-dependent way. More concretely, we first generate intermediate domains along the Wasserstein geodesic between two given consecutive domains in a feature space, then apply gradual self-training to adapt the source-trained classifier to the target along the sequence of intermediate domains. Empirically, we demonstrate that our GOAT framework can improve the performance of standard GDA when the given intermediate domains are scarce, significantly broadening the real-world application scenarios of GDA. Our code is available at \href{https://github.com/uiuctml/GOAT}{\texttt{https://github.com/uiuctml/GOAT}}.
\end{abstract}

\section{Introduction}

Modern machine learning models suffer from data distribution shifts across various settings and datasets~\citep{gulrajani2021in,sagawa2021extending,koh2021wilds,hendrycks2021natural,wiles2022a}, i.e., trained models may face a significant performance drop when the test data come from a distribution largely shifted from the training data distribution. Unsupervised domain adaptation (UDA) is a promising approach to address the distribution shift problem by adapting models from the training distribution (source domain) with labeled data to the test distribution (target domain) with unlabeled data~\citep{ganin2016domain,long2015learning,zhao2018adversarial,tzeng2017adversarial}. Typical UDA approaches include adversarial training~\citep{ajakan2014domain,ganin2016domain,zhao2018adversarial}, distribution matching~\citep{zhang2019bridging,tachet2020domain,li2021learning,li2022invariant}, optimal transport~\citep{courty2016optimal,courty2017joint}, and self-training (aka pseudo-labeling)~\citep{liang2019distant,liang2020we,zou2018unsupervised,zou2019confidence,wang2022understanding}. However, as the distribution shifts become large, these UDA algorithms suffer from significant performance degradation~\citep{kumar2020understanding,sagawa2021extending,abnar2021gradual,wang2022understanding}. This empirical observation is consistent with theoretical analyses~\citep{ben-david2010theory,zhao2019domain,tachet2020domain}, which indicate that the expected test accuracy of a trained model in the target domain degrades as the distribution shift becomes larger.


\begin{figure*}[tb]
    \centering
    \includegraphics[width=0.8\textwidth]{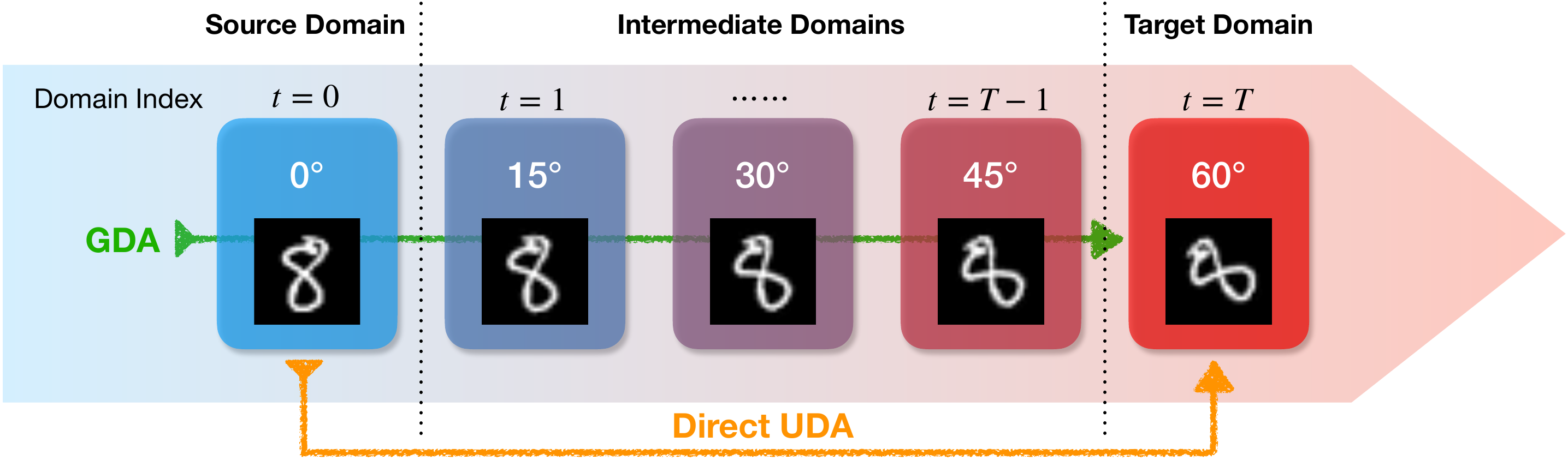}
    \caption{A schematic diagram comparing Unsupervised Domain Adaptation (UDA) vs. Gradual Domain Adaptation (GDA), using the example of Rotated MNIST. In GDA, given labeled data from a source domain, models are adapted to the target domain, with the help of unlabeled data from intermediate domains gradually shifting from the source to target.}
    \label{fig:UDA-vs-GDA}
\end{figure*}

When facing a large data distribution shift, our key strategy is the classic \emph{divide-and-conquer}: breaking the large shift into pieces of smaller shifts, resolving each piece with classic UDA approaches, and then combining all the intermediate solutions to recover a solution to the original data-shift problem (\Cref{fig:gda_algo}). Concretely, the data distribution shift between the source and target can be divided into pieces with intermediate domains bridging the two (i.e., the source and target). This methodology of leveraging intermediate data to tackle large distribution shift is known as gradual domain adaptation (GDA)~\citep{kumar2020understanding,abnar2021gradual,chen2021gradual,gadermayr2018gradual,wang2020continuously,bobu2018adapting,wulfmeier2018incremental,wang2022understanding}.

\begin{figure}[tb]
\centering
    \includegraphics[width=.8\columnwidth]{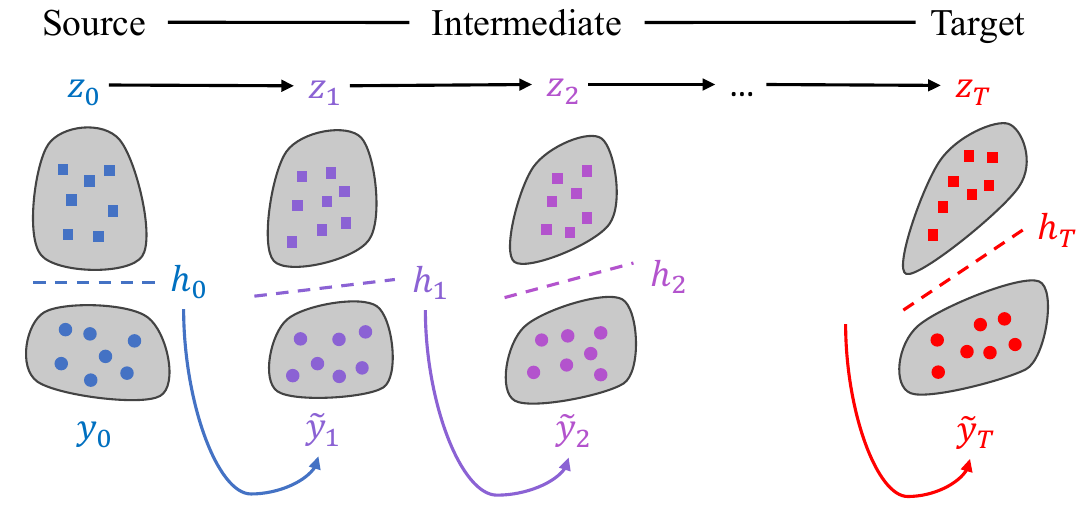}
\caption{An illustration of the divide-and-conquer strategy to address large data distribution shift (best viewed in color). The distribution shift between the source and target is divided into $T-1$ smaller pieces with (given or generated) unlabeled intermediate data. The model $h_t$ is gradually adapted in each step to reach the final solution.}
\label{fig:gda_algo}
\end{figure}

In the setting of GDA, where unlabeled intermediate data is available to the learner,~\citet{kumar2020understanding} proposed a simple yet effective algorithm, gradual self-training (GST), which applies self-training consecutively along the sequence of intermediate domains towards the target. \citet{kumar2020understanding} also proved an upper bound on the target error of GST, but it is pessimistic and unrealistic in practice. In particular, given source error $\eps_0$ and $T$ intermediate domains each with $n$ unlabelled data, the bound of \citet{kumar2020understanding} scales as $e^{\mathcal O(T)}\bigl(\eps_0+\cO \bigl(\sqrt {\log (T)/n}\bigr)\bigr)$, which grows exponentially in $T$. This indicates that the more intermediate domains for adaptation, the worse performance that gradual self-training would obtain in the target domain. In contrast, people have empirically observed that a relatively large $T$ is beneficial for gradual domain adaptation \citep{abnar2021gradual,chen2021gradual}. On the other hand, despite its simplicity, the self-training algorithm already exhibits some structures of the continual changing distributions:
\begin{obs}
    As the number of intermediate domains T increases, the accumulated error of the self-training algorithm also increases proportionally, due to the lack of ground-truth labels and the use of pseudo-labels.
\end{obs}
\begin{obs}
    As the number of intermediate domains T increases, by using the pseudo-labels, the effective sample size used by the self-training algorithm scales as $\cO(nT)$.
\end{obs}

Clearly, there is a fundamental tradeoff in the number of intermediate domains $T$ on the error of the self-training algorithm over the sequence of distributions. The existing generalization bound given by~\citet{kumar2020understanding} does not characterize this phenomenon. Furthermore, due to the exponential scaling factor, this upper bound becomes vacuous when $T$ is only moderately large. Based on the above two observations and the sharp gap between existing theory and empirical observations of gradual domain adaptation, we attempt to address the following important and fundamental questions:

\begin{quote}
\itshape
    For gradual domain adaptation, given the source domain and target domain, how does the number of intermediate domains impact the target generalization error? Is there an optimal choice for this number? If yes, then how to construct the optimal path of intermediate domains?
\end{quote}

To answer these questions, we first carry out a novel theoretical analysis on gradual self-training~\citep{kumar2020understanding}, then present a practical algorithm accordingly, which significantly outperforms vanilla gradual self-training. 
For the theoretical analysis, our setting is more general than that of \citet{kumar2020understanding}, in the sense that i) we have a milder assumption on the distribution shift, ii) we put almost no restriction on the loss function, and iii) our technique applies to all the $p$-Wasserstein distance metrics. As a comparison, existing analysis is restricted to ramp loss\footnote{Ramp loss can be seen as a truncated hinge loss so that it is bounded and more amenable for technical analysis.} and only applies to the $\infty$-Wasserstein metric. At a high level, we first focus on analyzing a pair of consecutive domains, and upper bound the error difference of any classifier over domains bounded by their $p$-Wasserstein distance; then, we telescope this lemma to the entire path over a sequence of domains, and finally obtain an error bound for gradual self-training: $\eps_0\mathrm{+}\cO\bigl(T\Delta \mathrm{+} \frac{T}{\sqrt{n}}\bigr) \mathrm{+} \widetilde{\mathcal O}\bigl(\frac{1}{\sqrt{nT}}\bigr)$, where $\Delta$ is the average $p$-Wasserstein distance between consecutive domains. We summarize the improvement of our analysis compared with~\citet{kumar2020understanding} in \Cref{Tab:theory}. 

\begin{table}
\centering
\captionof{table}{\small Comparison between our theoretical analysis and \cite{kumar2020understanding}. Our analysis is applicable to a more general setting and the generalization error bound is exponentially tighter in terms of the dependency on $T$.}
\begin{tabular}{l cc}
\toprule
 & \cite{kumar2020understanding} & Our Result\\
\midrule
Applicable Loss Functions & Ramp loss & All $\rho-$Lipschitz losses\\
Applicable Distance Metrics & $\infty-$Wasserstein metric & All $p-$Wasserstein metrics\\
Generalization Error Bound & $e^{\mathcal O(T)}\bigl(\eps_0+\cO \bigl(\sqrt {\log (T)/n}\bigr)\bigr)$ & $\eps_0\mathrm{+}\cO\bigl(T\Delta \mathrm{+} \frac{T}{\sqrt{n}}\bigr) \mathrm{+} \widetilde{\mathcal O}\bigl(\frac{1}{\sqrt{nT}}\bigr)$ \\
\bottomrule     
\label{Tab:theory}           
\end{tabular}
\end{table}

Interestingly, our bound indicates the existence of an optimal choice of $T$ that minimizes the generalization error, which could explain the success of moderately large $T$ used in practice. Notably, the $T\Delta$ in our bound could be interpreted as the length of the path of intermediate domains bridging the source and target, suggesting that one should also consider minimizing the path length $T \Delta$ in practices of gradual domain adaptation. For example, given fixed source and target domains, the path length $T\Delta$ is minimized as the intermediate domains are distributed along the Wasserstein geodesic between the source domain and target domain. 

The above insight is particularly helpful under the situation where intermediate domains are missing or scarce, which is often the case in real-world applications. It inspires a natural method to generate more intermediate domains useful for GDA. Based on this finding, we propose Generative Gradual Domain Adaptation with Optimal Transport (GOAT). At a high-level, GOAT contains the following steps:
\begin{enumerate}[label=\roman*]
    \item Generate intermediate domains ($z_t$ in \Cref{fig:gda_algo}) between each pair of consecutive given domains along the Wasserstein geodesic in a feature space.
    \item Apply gradual self-training (GST) over the sequence of given and generated domains. This produces a sequence of models $h_t$ and pseudo-labels $\tilde{y}_t$ as demonstrated in \Cref{fig:gda_algo}.
\end{enumerate}




Empirically, we conduct experiments on Rotated MNIST, Color-Shift MNIST, Portraits \citep{portraits} and Cover Type \citep{CoverType}, four benchmark datasets commonly used in the literature of GDA. The experimental results show that our GOAT significantly outperforms vanilla GDA, especially when the number of given intermediate domains is small. The empirical results also confirm the theoretical insights: i) when the distribution shift between a pair of consecutive domains is large, one can generate more intermediate domains to further improve the performance of GDA; ii) there exists an optimal choice for the number of generated intermediate domains.
\section{Preliminaries}\label{sec:prelim}
\paragraph{Notation}
$\cX,\cY$ denote the input and the output space, and $X,Y$ denote random variables taking values in $\cX,\cY$. In this work, each domain has a data distribution $\mu$ over $\cX \times \cY$, thus it can be written as $\mu = \mu(X,Y)$. When we only consider samples and disregard labels, we use $\mu(X)$ to refer to the sample distribution of $\mu$ over the input space $\cX$.

\subsection{Problem Setup}\label{sec:setup}
\paragraph{Binary Classification}
In the theoretical anlysis, we focus on binary classification with labels $\{-1,1\}$. Also, we consider $\mathcal Y$ as a compact space in $\bR$.
\paragraph{Gradually shifting distributions} We have $T\mathrm{+}1$ domains indexed by $\{0,1,...,T\}$, where domain $0$ is the source domain, domain $T$ is the target domain and domain $1,\dots,T\mathrm{-}1$ are the intermediate domains. These domains have distributions over $\cX \times \cY$, denoted as $\mu_{0}, \mu_{1}, \ldots, \mu_{T}$.
\paragraph{Classifier and Loss}
Consider the hypothesis class as $\mathcal H $ and the loss function as $\ell$.
We define the population loss of classifier $h\in \mathcal H$ in domain $t$ as
\begin{align*}
    \eps_{t}(h)\equiv \eps_{\mu_t}(h) \triangleq  \E_{\mu_t}[\ell (h(X),~Y)] = \E_{X,Y \sim \mu_t}[\ell (h(X),~Y)]
\end{align*}

\paragraph{Unsupervised Domain Adaptation (UDA)}
In UDA, we have a source domain and a target domain. During the training stage, the learner can access $m$ labeled samples from the source domain and $n$ unlabeled samples from the target domain. In the test stage, the trained model will then be evaluated by its prediction accuracy on samples from the target domain. The objective for UDA is to find the classifier $h^\star$ which minimizes the loss on the target domain
\begin{align}\label{eq:uda_obj}
    h^\star=\argmin_{h\in\cH} \E_{X,Y \sim \mu_t}[\ell (h(X),~Y)].
\end{align}

\paragraph{Gradual Domain Adaptation (GDA)}
Most UDA algorithms adapt models from the source to target in a one-step fashion, which can be challenging when the distribution shift between the two is large. Instead, in the setting of GDA, there exists a sequence of additional $T-1$ unlabeled intermediate domains bridging the source and target. We denote the underlying data distributions of these intermediate domains as $\mu_1(X,Y),\dots,\mu_{T-1}(X,Y)$, with $\mu_0(X, Y)$ and $\mu_T(X, Y)$ being the source and target domains, respectively. In this case, for each domain $t\in \{1,\dots,T\}$, the learner has access to $S_t$, a set of $n$ unlabeled data drawn i.i.d. from $\mu_t(X)$. Same as UDA, the goal of GDA is still to make accurate predictions on test data from the target domain (Eq. \ref{eq:uda_obj}), while the learner can train over $m$ labeled source data and $nT$ unlabeled data from $\{S_t\}_{t=1}^T$. To contrast the setting of UDA and GDA, we provide an illustration in Fig.~\ref{fig:UDA-vs-GDA} that compares UDA with GDA, using the Rotated MNIST dataset as an example.

We make a mild assumption on the input data below, which can be easily achieved by data preprocessing. This assumption is common in machine learning theory works~\citep{cao2019generalization,fine-grained,rakhlin2014notes}.
\begin{restatable}[Bounded Input Space]{as}{InputBound}
\label{assum:input-bound}
Consider the input space $\mathcal X$ is compact and bounded in the $d$-dimensional unit $L_2$ ball, i.e., $\cX \subseteq \{x \in \bR^d: \|x\|_{2}\leq 1\}$. 
\end{restatable}
This assumption effectively normalizes the input space and eliminates explicit dependence on the input dimension $d$ from our bounds\footnote{If we instead assumed $\|x\|_2 \leq \sqrt{d}$, which would correspond to $\mathcal{X} \subseteq [0,1]^d$, the constant $B$ in \Cref{assum:bounded-complexity} would gain a factor of $\sqrt{d}$.}.

To quantify distribution shifts between domains, we adopt the well-known Wasserstein distance metric in the Kantorovich formulation~\citep{kantorovich1939}, which is widely used in the optimal transport literature~\citep{villani2009optimal}.
\begin{restatable}[$p$-Wasserstein Distance]{defi}{WassersteinDistance} Consider two measures $\mu$ and $\nu$ over $\mathbb S \subseteq \bR^d$. For any $p\geq 1$, given a distance metric $d$, their $p$-Wasserstein distance is defined as 
\begin{align}
    W_{p}(\mu, \nu):=\left(\inf _{\gamma \in \Gamma(\mu, \nu)} \int_{\mathbb S \times \mathbb S} d(x, y)^{p} \mathrm{~d} \gamma(x, y)\right)^{1 / p}
\end{align}
where $ \Gamma(\mu, \nu)$ is the set of all measures over $\mathbb S \times \mathbb S$ with marginals equal to $\mu$ and $\nu$ respectively.
\end{restatable}

In this paper, we consider $p$ as a preset constant satisfying $p\geq 1$. Then, we can use the $p$-Wasserstein metric to measure the distribution shifts between consecutive domains.

\begin{restatable}[Distribution Shifts]{defi}{DistributionShifts} For $t=1,\dots,T$, denote
\begin{align}
    \Delta_t = W_p(\mu_{t-1}, \mu_{t})
\end{align}
Then, we define the average of distribution shifts between consecutive domains as
\begin{align}
    \Delta = \frac{1}{T} \sum_{t=1}^{T} \Delta_t
\end{align}
\end{restatable}

\paragraph{Remarks on Wasserstein Metrics} The $p$-Wasserstein metric has been widely adopted in many sub-areas of machine learning, such as generative models~\citep{WGAN,WAE} and domain adaptation~\citep{courty2014domain,courty2016optimal,courty2017joint,redko2019optimal}. Most of these works use $p=1$ or $2$, which is known to be good at quantifying many real-world data distributions~\citep{peyre2019computational}. However,  the analysis in~\citet{kumar2020understanding} only applies to $p=\infty$, which is uncommon in practice and can lead to a loose upper bound due to the monotonicity property of $W_p$. Since $W_\infty$ distance focuses on the maximum transportation cost between the measures, it is more prone to unboundedness, making it a less robust choice compared with $W_1$ and $W_2$.


\subsection{Gradual Self-Training}

The vanilla self-training algorithm (denoted as $\mathrm{ST}$) adapts classifier $h$ with empirical risk minimization (ERM) over pseudo-labels generated on an unlabelled dataset $S$, i.e.,
\begin{align}\label{eq:ST}
    h' = \mathrm{ST}(h,S) = \argmin_{f\in \mathcal H} \sum_{x\in S} \ell(f(x), h(x))
\end{align}
where $h(x)$ represents pseudo-labels provided by the trained classifier $h$, and $h'$ is the new classifier fitted to the pseudo-labels. The technique of hard labelling (i.e., converting $h(x)$ to one-hot labels) is used in some practices of self-training~\citep{xie2020self,van2020survey}, which can be viewed as adding a small modification to the loss function $\ell$.

Gradual self-training~\citep{kumar2020understanding}, applies self-training to the intermediate domains and the target domain successively, i.e., for $t = 1,\dots, T$,
\begin{align}\label{eq:gradual-ST}
    h_t = \mathrm{ST}(h_{t-1},S_t) = \argmin_{f\in \mathcal H} \sum_{x\in S_t} \ell(f(x), h_{t-1}(x))
\end{align}
where $h_0$ is the model fitted on the source data. $h_T$ is the final trained classifier that is expected to enjoy a low population error in the target domain, i.e., $\eps_{T}$.

Intuitively, one can expect that when the distribution shift between each consecutive pair of intermediate domains is large, the quality of the pseudo-labels obtained from the previous classifier can degrade significantly, hence hurting the final target generalization. This scenario is particularly relevant when the number of given intermediate domains is relatively small. 
\section{Theoretical Analyses}
In this section, we theoretically analyze gradual self-training under assumptions more relaxed than~\citet{kumar2020understanding}, and obtain a significantly improved error bound. Our theoretical analysis is roughly split into two steps: i) we focus on a pair of arbitrary consecutive domains with bounded distributional distance, and upper bound the prediction error difference of any classifier in the two domains by the distributional distance (\cref{lemma:error-diff}); ii) we view gradual self-training from an online learning perspective, and adopt tools in the online learning literature to analyze the algorithm together with results of step (i), leading to an upper bound (\cref{thm:gen-bound}) of the target generalization error of gradual self-training. Notably, our bound provides several profound insights on the optimal path of intermediate domains used in gradual domain adaptation (GDA), and also sheds light on the design of GDA algorithms. The proofs of all theoretical statements are provided in \cref{supp:proof}.

\subsection{Error Difference over Distribution Shift}\label{sec:error-diff}

Intuitively, gradual domain adaptation (GDA) splits the large distribution shift between the source domain and target domain into smaller shifts that are segmented by intermediate domains. Thus, in the view of reductionism~\citep{anderson1972more}, one should understand what happens in a pair of consecutive domains in order to comprehend the entire GDA mechanism.

To start, we adopt three assumptions from the prior work~\citep{kumar2020understanding}\footnote{Assumption \ref{assum:Lipschitz-loss} is not explicitly made by~\citet{kumar2020understanding}. Instead, they directly assume the loss function to be ramp loss, which is a more strict assumption than our Assumption \ref{assum:Lipschitz-loss}.}.

\begin{restatable}[$R$-Lipschitz Classifier]{as}{LipschitzModel}\label{assum:Lipschitz-model}
We assume each classifier $h \in \cH$ is $R$-Lipschitz in $\ell_2$ norm, i.e., $\forall x,x' \in \cX$,
\begin{align*}
    |h(x) - h(x')| \leq R \|x-x'\|_2
\end{align*}
\end{restatable}

\begin{restatable}[$\rho$-Lipschitz Loss]{as}{LipschitzLoss}\label{assum:Lipschitz-loss}
We assume the loss function $\ell$ is $\rho$-Lipschitz, i.e., $\forall y,y' \in \cY$,
\begin{align}\label{eq:Lipschitz-loss-assum}
    |\ell(y,\cdot) - \ell(y',\cdot)| \leq \rho \|y-y'\|_2\\ 
    |\ell(\cdot,y) - \ell(\cdot,y')| \leq \rho \|y-y'\|_2
\end{align}
\end{restatable}

\begin{restatable}[Bounded Model Complexity]{as}{BoundedComplexity}
\label{assum:bounded-complexity} \footnote{This assumption is actually reasonable and not strong. For example, under Assumption \ref{assum:input-bound} and \ref{assum:Lipschitz-model}, linear models directly satisfy \eqref{eq:rademacher-assum}, as proved in~\citep{kumar2020understanding,liang2016cs229t}.}
    We assume the Rademachor complexity~\citep{bartlett2002rademacher}, $\cR$, of the hypothesis class, $\cH$, is bounded for any distribution $\mu$ considered in this paper. That is, for some constant $B > 0$,
    \begin{align}\label{eq:rademacher-assum}
        \cR_n(\cH; \mu) = \E\left[\sup_{h \in \cH}\frac{1}{n} \sum_{i=1}^n \sigma_i h(x_i) \right]\leq \frac{B}{\sqrt{n}}
    \end{align}
    where the expectation is w.r.t. $x_i \sim \mu(X)$ and $\sigma_i \sim \mathrm{Uniform}(\{-1,1\})$ for $i=1,\dots,n$.
\end{restatable}

With these assumptions, we can bound the population error difference of a classifier between a pair of shifted domains in the following proposition. The proof is in \cref{supp:proof:error-diff}.
\begin{restatable}[Error Difference over Shifted Domains]{lem}{errordiff}
\label{lemma:error-diff}
    Consider two arbitrary measures $\mu, \nu$ over $\cX \times \cY$. Then, for arbitrary classifier $h$ and loss function $\ell$ satisfying Assumption \ref{assum:Lipschitz-model}, \ref{assum:Lipschitz-loss}, the population loss of $h$ on $\mu$ and $\nu$ satisfies
    \begin{align}\label{eq:thm:additive-bound:main}
        |\eps_{\mu}(h) - \eps_{\nu}(h)| & \leq \rho \sqrt{R^2 + 1} ~W_p(\mu, \nu)
    \end{align}
    where $W_p$ is the Wasserstein-$p$ distance metric and $p\geq 1$.
\end{restatable}

Eq. \eqref{eq:gradual-ST} depicts each iteration of gradual self-training with an past classifier $h_t$ and a new one $h_{t+1}$, which are fitted to $S_{t}$ and $S_{t+1}$, respectively. Naturally, one might be curious about how well the performance of $h_{t+1}$ in domain $t\mathrm{+}1$ is compared with $h_{t}$ in domain $t$. We answer this question as follows, with proof in Appendix \ref{supp:proof:algo-stability}.
\begin{restatable}[The stability of the ST algorithm]{prop}{AlgoStability}\label{prop:algorithm-stability}
Consider two arbitrary measures $\mu,\nu$, and denote $S$ as a set of $n$ unlabelled samples i.i.d. drawn from $\mu$. Suppose $h\in \cH$ is a pseudo-labeler that provides pseudo-labels for samples in $S$. Define $\hat h \in \cH$ as an ERM solution fitted to the pseudo-labels,
\begin{align}
    \hat h = \argmin_{f \in \mathcal{H}} \sum_{x\in S} \ell(f(x), h(x))
\end{align}
Then, for any $\delta \in (0,1)$, the following bound holds true with probability at least $1-\delta$,
\begin{align}\label{eq:algo-stability}
    \bigl|\eps_{\mu}(\hat h) \mathrm{-}\eps_{\nu} (h) \bigl| \leq \cO\biggl(W_p(\mu,\nu) \mathrm{+} \frac{\rho B\mathrm{+}\sqrt{\log\frac 1 \delta }}{\sqrt n}~\biggr)
\end{align}
\end{restatable}
\paragraph{Comparison with~\citet{kumar2020understanding}}
The setting of~\citet{kumar2020understanding} is more restrictive than ours. For example, its analysis is specific to ramp loss~\citep{huang2014ramp}, a rarely used loss function for binary classification.~\citet{kumar2020understanding} also studies the error difference over consecutive domains, and prove a multiplicative bound (in Theorem 3.2 of~\citet{kumar2020understanding}), which can be re-expressed in terms of our notations and assumptions as
\begin{align}\label{eq:kumar-multiplicative}
    \eps_{\mu}(\hat h) \leq \frac{2}{1\mathrm{-}R \Delta_\infty}\eps_{\nu}(h) \mathrm{+} \eps_\mu^* \mathrm{+} \cO\biggl(\frac{\rho B \mathrm{+} \sqrt {\log \frac 1 \delta}}{\sqrt n}\biggr)
\end{align}
where $\eps_\mu^* \triangleq \min_{f\in \cH}\eps_\mu(f)$ is the optimal error of $\cH$ in $\mu$, and $\Delta_\infty\triangleq \max_{y\in\{-1,1\}}(W_\infty(\mu(X|Y=y),\nu(X|Y=y)))$ can be seen as an analog to the $W_p(\mu,\nu)$ in \eqref{eq:algo-stability}.~\citet{kumar2020understanding} assumes $1-R\Delta_\infty>0$, thus the error $\eps_\nu(h)$ is increased by the factor $\frac{2}{1-R\Delta_\infty} > 1$ in the above error bound of $\eps_\mu(\hat h)$. This leads to a target domain error bound \textit{exponential} in $T$ (Corollary 3.3. of~\citet{kumar2020understanding}) when one applies \eqref{eq:kumar-multiplicative} to the sequence of domains iteratively in gradual self-training (i.e., Eq. \eqref{eq:gradual-ST}). In contrast, our \eqref{eq:algo-stability} indicates $\eps_{\mu}(\hat h ) \leq \eps_{\nu}(h) + \mathrm{other~terms}$, which increases the error $\eps_\nu(h)$ in an \textit{additive} way, leading to a target domain error bound \textit{linear} in $T$.

\paragraph{Remarks on Generality}
\cref{lemma:error-diff} and \cref{prop:algorithm-stability} are not restricted to gradual domain adaptation. Of independent interest, they can be leveraged as useful theoretical tools to handle distribution shifts in other machine learning problems, including unsupervised domain adaptation, transfer learning, out-of-distribution (OOD) robustness, and group fairness.

\subsection{An Online Learning View of GDA}\label{sec:online-view}

One can naively apply \cref{prop:algorithm-stability} to gradual self-training over the sequence of domains (i.e., Eq. \eqref{eq:gradual-ST}) iteratively and obtain an error bound of the target domain as
\begin{align}\label{eq:naive-gen-bound}
    \eps_{T}(h_{T}) \leq \eps_{0}(h_0) + \cO\biggl( T \Delta \mathrm{+} T\frac{\rho B\mathrm{+}\sqrt{\log\frac 1 \delta }}{\sqrt n}~\biggr)
\end{align}
Obviously, the larger $T$, the higher the error bound becomes (this holds even if one assumes $T\Delta\leq \mathrm{constant}$ for fixed source and target domains). However, this contradicts with empirical observations that a moderately large $T$ is optimal~\citep{kumar2020understanding,abnar2021gradual,chen2021gradual}.


To resolve this discrepancy, we take an online learning view of gradual domain adaptation, which allows us to obtain a more optimistic error bound. Specifically, we consider the domains $t=0,\dots,T$ arriving sequentially to the model. This process can be formalized as follows. For each domain $t=0,\dots,T$:
    \begin{enumerate}
        \item Observe unlabeled data $S_t = \{x_i^t\}_{i=1}^n$ from domain $\mu_t$.
        \item If $t=0$ (source domain), use true labels. Otherwise, generate pseudo-labels $\hat{y}_i^t = h_{t-1}(x_i^t)$ using the previous model $h_{t-1}$.
        \item Update the model: $h_t = \argmin_{f \in \mathcal{H}} \sum_{i=1}^n \ell(f(x_i^t), \hat{y}_i^t)$.
    \end{enumerate}

This online learning perspective allows us to leverage tools from sequential learning theory, particularly the framework of \citet{rakhlin2015online}, which views online binary classification as a process on a complete binary tree. By applying this view, we can utilize results on sequential Rademacher complexity (Definition \ref{def:seq-rademacher}) and the discrepancy measure between distributions (Definition \ref{def:disc}). The key advantage of this approach is that it enables us to obtain bounds that depend on the total number of samples $nT$, rather than just $T$ as in the naive approach. Specifically, terms of order $\mathcal{O}(\sqrt{1/T})$ in the naive bound become $\mathcal{O}(\sqrt{1/nT})$ in our improved bound. Moreover, this view allows us to better characterize how the error accumulates across domains, leading to the improved linear dependence on $T$ in our final bound (Theorem \ref{thm:gen-bound}), compared to the exponential dependence in previous work \citep{kumar2020understanding}.

To proceed, certain structural assumptions and complexity measures are necessary. For example, VC dimension~\citep{vapnik1999nature} and Rademacher complexity~\citep{bartlett2002rademacher} are proposed for supervised learning. Similarly, in online learning, Littlestone dimension~\citep{littlestone1988learning}, sequential covering number~\citep{rakhlin2010online} and sequential Rademacher complexity~\citep{rakhlin2010online,rakhlin2015online} are developed as useful complexity measures. To study gradual self-training in an online learning framework, we adopt the framework of~\citet{rakhlin2015online}, which views online binary classification as a process in the structure of a \textit{complete binary tree} and defines the \textit{sequential Rademacher complexity} upon that.

\begin{restatable}[Complete Binary Trees]{defi}{CompleteBinaryTree}\label{def:complete-binary-tree}
We define two complete binary trees $\mathscr X, \mathscr Y$, and the path $\bm \sigma$ in the trees:
\begin{itemize}[leftmargin=0em,align=left,noitemsep,nolistsep]
    \item[] $\mathscr X\triangleq(\mathscr X_0, ..., \mathscr X_{T})$, a sequence of mappings with $\mathscr{X}_t : \{\pm 1\}^{t} \rightarrow \cX$ for $t = 0,...,T$.
    \item[] $\mathscr Y\triangleq(\mathscr Y_0, ..., \mathscr Y_{T})$, a sequence mappings with $\mathscr{Y}_t : \{\pm 1\}^{t} \rightarrow \cY$ for $t = 0,...,T$.
    \item[] $\bm \sigma = (\sigma_0, ..., \sigma_{T}) \in \{\pm 1\}^{t} $, a path in $\mathscr{X}$ or $\mathscr{Y}$.
\end{itemize}
\end{restatable}

\begin{restatable}[Sequential Rademacher Complexity]{defi}{DefSeqRadamacher}\label{def:seq-rademacher}~
Consider $\bm \sigma$ as a sequence of Rademacher random variables and a $t$-dimensional probability vector $\mathbf{q}_t= (q_0,...,q_{t-1})$, then the sequential Rademacher complexity of $\cH$ is
\begin{align*} 
    \mathcal{R}^{\mathrm{seq}}_{t}(\mathcal H) &= \sup_{\bm{\mathscr X,\mathscr Y}} \E_{\bm\sigma}\left[ \sup_{h \in \mathcal H}\sum_{\tau=0}^{t-1} \sigma_\tau q_\tau \ell\bigl( h( { \mathscr X}_\tau (\bm{\sigma})), \mathscr Y_\tau(\bm{\sigma}) \bigr)\right] 
\end{align*}
\end{restatable}
To better understand this measure, we present examples of two common model classes\footnote{The probability vector $\mathbf{q}_t$ is taken to be uniform in these cases.}, which are provided in~\citet{rakhlin2014notes}.

\begin{example}[Linear Models] \label{example:linear-model}
For the linear model class that is $R$-Liphschtiz, i.e., $\cH = \{ x \rightarrow w^\top x: \|w\|_2 \leq R\}$, we have $\mathcal{R}^{\mathrm{seq}}_t(\mathcal H) \leq \frac{R}{\sqrt{t}}$ for $t\in \mathbb{Z}_+$.

\end{example}

\begin{example}[Neural Networks]\label{example:neural-net}
Consider $\mathcal H$ as the hypothesis class of $R$-Lipschitz $L$-layer fully-connected neural nets with $1$-Lipschitz activation function (e.g., ReLU, Sigmoid, TanH). Then, its sequential Rademacher complexity is bounded as $\mathcal{R}^{\mathrm{seq}}_t(\mathcal H) \leq \cO\left(R \sqrt{\frac{\left(\log t \right)^{3(L-1)}}{t}}\right)$ for $t\in \mathbb{Z}_+$.
\end{example}
Besides the model complexity measure, we also adopt a measure of discrepancy among multiple data distributions, which is proposed in works of online learning for time-series data~\citep{kuznetsov2014generalization,kuznetsov2015learning,kuznetsov2016time,kuznetsov2017generalization,kuznetsov2020discrepancy}.
\begin{restatable}[Discrepancy Measure]{defi}{Disc}\label{def:disc}
For any $t$-dimensional probability vector $\bq_{t} = (q_0,...,q_{t-1})$, the discrepancy measure $\disc(\bq_{t})$ is defined as 
\begin{align}\label{eq:disc-def}
    \disc(\bq_{t}) &= \sup_{h\in \cH} \left( \eps_{t-1}(h) - \sum_{\tau=0}^{t-1} q_\tau \cdot \eps_{\tau}(h) \right) 
\end{align}
\end{restatable}
Intuitively, this discrepancy measure quantifies the maximum difference between the error of a hypothesis on the last domain ($\varepsilon_{t-1}(h)$) and a weighted average of its errors on all previous domains ($\sum_{\tau=0}^{t-1} q_\tau \cdot \varepsilon_\tau(h)$). This measure captures how much the ``difficulty'' of the learning problem can change across domains. In the context of gradual domain adaptation, a small discrepancy suggests that the domains are changing gradually, making it easier for the model to adapt. Conversely, a large discrepancy indicates significant shifts between domains, which could make adaptation more challenging. The supremum over $\mathcal{H}$ in the definition ensures that we're considering the worst-case scenario across all possible hypotheses in our model class. This conservative approach helps us derive bounds that hold regardless of which specific hypothesis our learning algorithm might choose. 

We can further bound this discrepancy in our setting (defined in Sec. \ref{sec:prelim}) as follows. The proof is in \cref{supp:proof:discrepancy-bound}. 

\begin{restatable}[Discrepancy Bound]{lem}{DiscBound}\label{lemma:disc-bound}
With \cref{lemma:error-diff}, the discrepancy measure \eqref{eq:disc-def} can be upper bounded as
\begin{align}\label{eq:desc-upper-bound-1}
    \disc(\bq_{t}) &\leq \rho \sqrt{R^2 + 1} \sum_{\tau=0}^{t-1}q_\tau  (t-\tau - 1)\Delta
\end{align}
With $\bq_{t} = \bq_{t}^* = (\frac{1}{t},..., \frac{1}{t})$, this upper bound can be minimized as
\begin{align}
    \disc(\bq_{t}^*) \leq  \rho \sqrt{R^2 + 1} ~t \Delta / 2= \cO(t\Delta)
\end{align}
\end{restatable}

\subsection{Generalization Bound for Gradual Self-Training}
With our results obtained in \cref{sec:error-diff} and tools introduced in \cref{sec:online-view}, we can prove a generalization bound for gradual self-training within online learning frameworks such as~\citet{kuznetsov2016time,kuznetsov2020discrepancy}. However, if we use these frameworks in an off-the-shelf way, the resulting generalization bound will have multiple terms with dependence on $T$ and no dependence on $n$ (the number of samples per domain), since these online learning works do not care about the data size of each domain. This will cause the resulting bound to be loose in terms of $n$. To resolve this, we come up with a novel reductive view of the learning process of gradual self-training, which is more fine-grained than the original view in~\citet{kumar2020understanding}. This reductive view enables us to make the generalization bound to depend on $n$ in an intuitive way, which also tightens the final bound. We defer explanations of this view to \cref{supp:proof:gen-bound} along with the proof of \cref{thm:gen-bound}. 

Finally, we prove a generalization bound for gradual self-training that is much tighter than that of~\citet{kumar2020understanding}.
\begin{restatable}[Generalization Bound for Gradual Self-Training]{thm}{GenBound}
\label{thm:gen-bound} 
For any $\delta\in(0,1)$, the population loss of gradually self-trained classifier $h_{T}$ in the target domain is upper bounded with probability at least $1-\delta$ as
\begin{align}
 \eps_{T} (h_{T}) &\leq \sum_{t=0}^T q_t \eps_{t}(h_t) + \|\bq_{\scriptscriptstyle n(T+1)}\|_2 \left(1\mathrm{+}\cO\left(\sqrt{\log   (1/ \delta)}\right)\right)\nonumber\\
 &\qquad \mathrm{+}\disc(\bq_{\scriptscriptstyle T+1})\mathrm{+}\cO\left(\sqrt{\log T } \cR_{n(T+1)}^{\mathrm{seq}}(\ell \circ \cH)\right) \label{eq:gen-bound-general}
 \end{align}
 For the class of neural nets considered in Example \ref{example:neural-net}, 
 \begin{align}
 \eps_{T}(h_T) &\leq  \eps_{0}(h_0) +\cO\biggl(T\Delta \mathrm{+}\frac{T}{\sqrt n} \mathrm{+} T\sqrt{\frac{\log 1 / \delta}{n}}\mathrm{+}\frac{1}{\sqrt{nT}}  \mathrm{+}\sqrt{\frac{(\log nT)^{3L-2}}{nT}} \mathrm{+} \sqrt{\frac{\log 1/\delta}{nT}}~\biggr)\label{eq:gen-bound}
\end{align}
\end{restatable}
\paragraph{Remark}
The bound in Eq.~\eqref{eq:gen-bound} is rather intuitive\footnote{Eq.~\eqref{eq:gen-bound} is derived from Eq.~\eqref{eq:gen-bound-general} by substituting the expression for $\cR_{n(T+1)}^{\mathrm{seq}}(\ell \circ \cH)$ from Example \ref{example:neural-net}.}: the first term $\eps_{0}(h_0)$ is the source error of the initial classifier, and $T\Delta$ corresponds to the total length of the path of intermediate domains connecting the source domain and the target domain. The asymptotic $\cO(T/\sqrt{n})$ term is due to the accumulated estimation error of the pseudo-labeling algorithm incurred at each step. The $\cO(1/\sqrt{nT})$ term characterizes the overall sample size used by the algorithm along the path, i.e., the algorithm has seen $n$ samples in each domain, and there are $T$ total domains that gradual self-training runs on.

\paragraph{Comparison with~\citet{kumar2020understanding}}
Using our notation, the generalization bound of~\citet{kumar2020understanding} can be re-expressed as
\begin{align}\label{eq:kumar-gen-bound}
    \eps_{T}(h_T) \leq e^{\cO(T)} \biggl(\eps_{0}(h_0) \mathrm{+} \cO\bigl(\frac{1}{\sqrt n } \mathrm{+} \sqrt{\frac{\log T}{n}}~\bigr)\biggr),
\end{align}
which grows \textit{exponentially} in $T$ as a multiplicative factor. In contrast, our bound \eqref{eq:gen-bound} grows only additively and linearly in $T$, achieving an \textit{exponential improvement} compared with the bound of~\citet{kumar2020understanding} shown in \eqref{eq:kumar-gen-bound}.

\begin{figure}[tb]
\begin{center}
\centerline{\includegraphics[width=0.7\columnwidth]{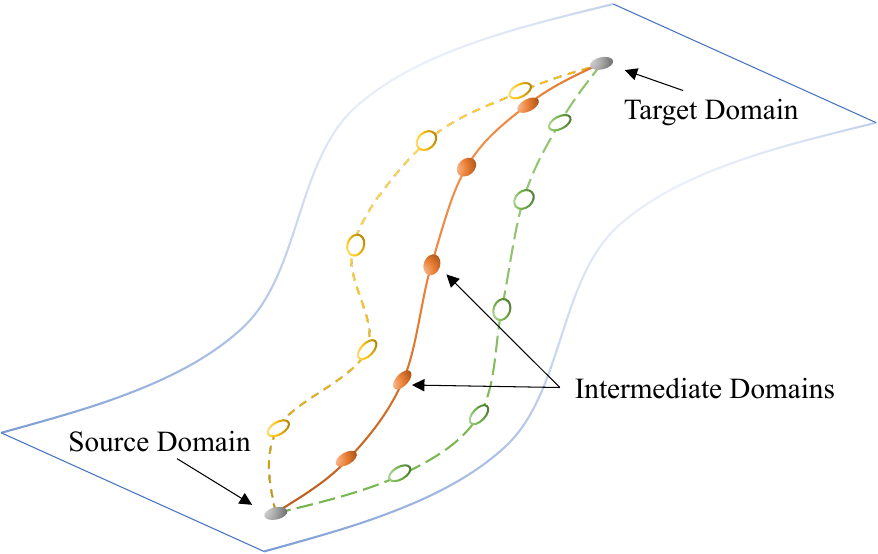}}
\caption{
An illustration of the optimal path in gradual domain adaptation, with a detailed explanation in Sec.~\ref{sec:optimal-path}. The orange path is the geodesic connecting the source domain and target domain.  
}
\label{fig:manifold-demo}
\end{center}
\end{figure}

\subsection{Optimal Path of Gradual Self-Training}\label{sec:optimal-path}
It is worth pointing out that our generalization bound in Theorem~\ref{thm:gen-bound} applies to any path connecting the source domain and target domain with $T$ steps, as long as $\mu_0$ is the source domain and $\mu_T$ is the target domain. In particular, if we define $\deltam$ to be an upper bound\footnote{Need to be large enough to ensure that $\mathcal P$ is non-empty.} on the average $W_p$ distance between any pair of consecutive domains along the path, i.e., $\Delta_{\max} \geq \frac{1}{T}\sum_{t=1}^T W_p(\mu_{t-1}, \mu_t)$, and let $\mathcal{P}$ to be the collection of paths with $T$ steps connecting $\mu_0$ and $\mu_T$:
\begin{equation*}
\mathcal{P}\defeq \{(\mu_{t})_{t=0}^T\mid \frac 1 T \sum_{t=1}^T W_p(\mu_{t-1}, \mu_t)\leq \deltam\}~,
\end{equation*}
then we can extend the generalization bound in Theorem~\ref{thm:gen-bound}:
\begin{equation}
\eps_{T}(h_T) \leq  \eps_{0}(h_0) \mathrm{+} \inf_{\mathcal{P}} \widetilde{\cO}\biggl(T\deltam \mathrm{+}\frac{T}{\sqrt n} \mathrm{+} \sqrt{\frac{1}{nT}}~\biggr)
\label{eq:simplified}
\end{equation}
Minimizing the RHS of the above upper bound w.r.t.\ $T$ (the proof is provided in \cref{supp:proof:optimal-T}), we obtain the optimal choice of $T$ on the order of
\begin{align}\label{eq:optimal-T-expression}
    \widetilde{\cO}\left(\left(\frac{1}{1 + \deltam \sqrt{n}}\right)^{2/3}\right).
\end{align}
However, the above asymptotic optimal length may not be achievable, since we need to ensure that $T\deltam$ is at least the length of the geodesic connecting the source domain and target domain. To this end, define $L$ to be the $W_p$ distance between the source domain and target domain, we thus have the optimal choice $T^*$ as
\begin{equation}
    T^* = \max\left\{\frac{L}{\deltam}, \widetilde{\cO}\left(\left(\frac{1}{1 + \deltam \sqrt{n}}\right)^{2/3}\right)\right\}.
\label{equ:optt}
\end{equation}
Intuitively, the inverse scaling of $T^*$ and $\deltam$ suggests that, if the average distance between consecutive domains is large, it is better to take fewer intermediate domains. 

\paragraph{Illustration of the Optimal Path}
To further illustrate the notion of the optimal path connecting the source domain and target domain implied by our theory, we provide an example in Fig.~\ref{fig:manifold-demo}. Consider the metric space induced by $W_p$ over all the joint distributions with finite $p$-th moment, where both the source and target could be understood as two distinct points. In this case, there are infinitely many paths of step size $T$ connecting the source and target, such that the average pairwise distance is bounded by $\deltam$. Hence, one insight we can draw from Eq.~\eqref{eq:simplified} is that: if the learner could construct the intermediate domains, then it is better to choose the path that is as close to the geodesic, i.e., the shortest path between the source and target (under $W_p$), as possible. This key observation opens a broad avenue forward toward algorithmic designs of gradual domain adaptation to \textit{construct} intermediate domains for better generalization performance in the target domain.
\section{Generative Gradual Domain Adaptation with Optimal Transport}
Inspired by the theoretical results, in this section, we present our algorithm to automatically generate a series of intermediate domains between any pair of consecutive given domains, with the hope that when applied to the sequence of generated intermediate domains, GST could lead to better target generalization. Before presenting the proposed algorithm, we first formally introduce several notions that will be used in the design of our algorithm. 

The optimal transport problem was initially formalized by \citet{monge1781memoire}, and \citet{kantorovich1939} further relaxed the deterministic nature of Monge's problem formulation. In this part, we adopt Kantorovich's formulation of optimal transport~\citep{kantorovich1939}, which aims at finding the optimal coupling that minimizes a total transport cost.
\begin{restatable}[Optimal Coupling]{defi}{OptimalTransport} Given measures $\mu,\nu$ over $\mathcal X$ and a lower semi-continuous cost function\footnote{The existence of an optimal transport plan is contingent on the cost being lower semi-continuous. See, e.g., Proposition 2.1 from \citet{villani2021topics}.} $c:\mathcal X \times \mathcal X \mapsto [0,\infty)$, the optimal transport coupling $\gamma^*$ is the one that attains the infimum of the total transport cost: 
\begin{align}
    \inf_{\gamma \in \Gamma(\mu, \nu)} \int_{\mathcal X \times \mathcal X} c(x,x') d\gamma(x,x')~.
\end{align}
where $\Gamma(\mu,\nu)$ is the set of all probability measures on $\mathcal X \times \mathcal X$ with marginals as $\mu,\nu$.
\end{restatable}
One can create a path of measures that interpolates the given two, and the theory of optimal transport can help us find the optimal path that minimizes the path length measured under the Wasserstein metric, i.e., the sum of Wasserstein distances between each pair of consecutive measures along the path. This optimal path is termed the Wasserstein geodesic, which is formally defined below.
\begin{restatable}[Wasserstein Geodesic]{defi}{Geodesic}\label{def:geodesic}
Given two measures $\nu_0,\nu_1$ over $\mathcal X$ and an optimal coupling $\gamma^\star$. Let $\sharp$ denote the push-forward operator on measures. Then, a (constant-speed) Wasserstein geodesic between $\nu_0,\nu_1$ under Euclidean metric can be defined by the path $\mathcal{P} (\nu_0,\nu_1)\coloneqq \{(g_t)_\sharp\gamma^\star : t\in[0,1] \}$, where $g_t(x,y) =  (1-t)x + t y$.
\end{restatable}

\subsection{Motivations}

The target domain error bound of gradual self-training, i.e., Eq. \eqref{eq:gen-bound}, has a dominant term $T\Delta$, which can be interpreted as the length of the path of intermediate domains connecting the source and target. 
Interestingly, we find that this path is related to the \textbf{Wasserstein geodesic} between the source $\mu_0$ and target $\mu_T$, and we formalize our findings as follows. 

\begin{restatable}[Path Length of Intermediate Domains]{prop}{DomainPath}\label{prop:domain-path}
For arbitrary intermediate domains $\mu_1,\dots,\mu_{T-1}$, the following inequality holds,
\begin{align}\label{eq:prop:domain-path}
T\Delta = \sum_{t=1}^{T} W_p(\mu_{t-1}, \mu_{t}) \geq W_p(\mu_{0}, \mu_{T}),
\end{align}
where the equality is obtained if and only if the intermediate domains $\mu_1,\dots,\mu_{T-1}$ sequentially fall along the Wasserstein geodesic between $\mu_0$ and $\mu_T$.
\end{restatable}
Without explicit access to the intermediate domains, gradual domain adaptation cannot be applied. Interestingly, \Cref{prop:domain-path} sheds light on the task of intermediate domain generation to bridge this gap: \textit{the generated intermediate domains should fall on or close to the Wasserstein geodesic in order to minimize the path length.}

Note that in GDA, we cannot directly measure $\Delta$ since it requires access to the joint distributions of the intermediate domains, whereas only unlabeled data are available to us. In order to bridge the gap, in this paper, we make the following assumption of the intermediate domains.
\begin{restatable}[Feature Space]{as}{FeatureSpace}\label{assum:Feature-Space}
      There exists a feature space $\cZ$ such that the covariate shift assumption holds over $\cZ$. Specifically, the conditional distribution of $Y$ given the feature $Z$ is invariant across all the intermediate domains, i.e., for any two domains $i$ and $j$ with $i\neq j$, $\mu_i(Y|Z)=\mu_j(Y|Z)$. 
\end{restatable}
Note that covariate shift is one of the most common assumptions in the literature of domain adaptation~\citep{ben2007analysis,adel2017unsupervised,arjovsky2019invariant,redko2019optimal,zhao2019deep,rosenfeld2020risks,wang2022isr}. It is one way to ensure that the knowledge contained in different domains are inherently relevant such that the success of domain adaptation is possible~\citep{zhang2013domain}. It has been widely applied in various applications, including computer vision~\citep{adel2017unsupervised,arjovsky2019invariant,redko2019optimal,zhao2019deep}, natural language processing~\citep{ash2016unsupervised}, and robot control~\citep{akiyama2010efficient,Sugiyama2013LearningUN}.
Under this assumption, the Wasserstein distance between the joint distance $W_p\left(\mu_{t-1}(Z,Y), ~\mu_{t}(Z,Y)\right)$ reduces to the one between the marginal feature distribution $W_p\left(\mu_{t-1}(Z), ~\mu_{t}(Z)\right)$.

\subsection{Computation with Optimal Transport}\label{sec:algo:analysis}
From \Cref{def:geodesic}, we know that one has to solve an optimal transport problem to generate intermediate domains along the Wasserstein geodesic. As a first step, we consider the optimal transport between a source domain and a target domain.

\paragraph{Solve Optimal Transport with Linear Programming} In unsupervised domain adaptation (UDA), the source and target domains have finite training data. Hence, we can consider the measures of the source and target to be discrete, i.e., $\mu_0$ and $\mu_T$ only have probability mass over the finite training data points. More formally, denoting the source dataset as $S_0 = \{x_{0i}\}_{i=1}^{m}$ and target dataset as $S_T=\{x_{Tj}\}_{i=1}^{n}$, the empirical measures $\mu_0$ and $\mu_T$ can be expressed as
\begin{align}
    \mu_0 = \frac{1}{m}\sum_{i=1}^{m} \delta(x_{0i})~, \quad \mu_T = \frac{1}{n}\sum_{j=1}^{n} \delta(x_{Tj}),
\end{align}
where $\delta(x)$ represents the Dirac delta distribution at $x$~\citep{dirac1930principles}. Under the discrete case, the push-forward operator $\mathcal{T}^*$ that pushes $\mu_0$ forward to $\mu_T$ can be obtained by solving a linear program~\citep{peyre2019computational}.

\begin{restatable}{prop}{LP}\label{prop:LP}
Consider $\mu_0$ over source data $\{x_{0i}\}_{i=1}^{m}$ and $\mu_T$ over target data $\{x_{Tj}\}_{i=1}^{n}$. Given a transport cost function $c:\mathcal X \times \mathcal X \mapsto [0,\infty)$, there exists a randomized optimal transport map (induced from the optimal coupling $\gamma^*$), $\mathcal{T}^*$, which satisfies $\cT^*_\sharp\mu_0 = \mu_T$. Furthermore, for $i\in[m]$, $\cT^*$ maps $x_{0i}$ as follows,
\begin{align}
    \cT^*_\sharp\delta(x_{0i}) = \sum_{j=1}^n \gamma^*_{ij} \delta(x_{Tj}),
\end{align}
where $\gamma^*\in\mathbb{R}_{\geq 0 }^{m \times n}$ is the optimal transport plan, a non-negative matrix of dimension $m\times n$. The plan $\gamma^*$ can be obtained by solving the following linear program,
\begin{align}\label{eq:opt-plan}
\gamma^{*}=\argmin _{\gamma \in \mathbb{R}_{\geq 0}^{m \times n}} \sum_{i, j} \gamma_{i, j} c(x_{0i},x_{Tj})\\
\text { s.t. }~ \gamma \bm{1}_{n}= \frac{1}{m} \bm{1}_{m}~\text{ and }~ \gamma^{T} \bm{1}_m = \frac{1}{n}\bm{1}_n\nonumber
\end{align}
\end{restatable}

\paragraph{Generating Intermediate Domains with Optimal Transport} \Cref{prop:LP} demonstrates that one can use linear programming (LP) to solve the optimal transport problem between a source dataset and a target dataset. With the optimal transport plan $\gamma^*$, one can leverage \Cref{def:geodesic} to generate intermediate domains along the Wasserstein geodesic. Specifically, for $t=1,\dots,T-1$, the measure of the intermediate domain $t$ can be obtained by the following push-forward
\begin{align}\label{eq:inter-domain-pushforward}
    \mu_{t}\mathrm{=}\left( \frac{T\mathrm{-}t}{T} \Id + \frac{t}{T} \cT^*\right)_\sharp \mu_0 \mathrm{=} \frac 1 m \sum_{i,j} \gamma^*_{ij} \delta \left(\frac{T\mathrm{-}t}{T} x_{0i} \mathrm{+} \frac{t}{T} x_{Tj}\right)
\end{align}
Intuitively, $\mu_t$ can be interpreted as a discrete measure over $n_{\gamma^*}$ data points with data weights assigned by $\gamma^*_{ij}$, where $n_{\gamma^*} \coloneqq \sum_{i,j} \mathds{1}[\gamma_{ij} > 0]$ is the number non-zero entries in the matrix $\gamma^*$.

\paragraph{Space Complexity} Clearly, one needs to store the optimal transport plan matrix $\gamma^*\in \bR_{\geq 0}^{m \times n}$, in order to generate intermediate domains with \eqref{eq:inter-domain-pushforward}. Thus, the space complexity appears to be $\cO(mn)$. However, by leveraging the theory of linear programming, one can show that the maximum number of non-zero elements of the solution $\gamma^*$ to the LP \eqref{eq:opt-plan} is at most $m+n-1$~\citep{peyre2019computational}. Thus, the space complexity can be reduced to $\cO(m+n)$ when using a sparse matrix format to store $\gamma^*$.

\paragraph{Time Complexity} For simplicity, let us consider $m=n$. Then, the time complexity of solving the LP \eqref{eq:opt-plan} is known to be $O(n^3\log(n))$~\citep{cuturi2013sinkhorn,pele2009fast}.


\subsection{Proposed Algorithm}
\begin{algorithm}[t!]

\caption{Generative Gradual Domain Adaptation with Optimal Transport (GOAT)}\label{algo:main}
\begin{algorithmic}
\REQUIRE {$S_0^X=\{x_{0i}\}_{i=1}^{m}$, $S_T^X=\{x_{Ti}\}_{i=1}^{n}$; Encoder $\mathcal{E}$; Source-trained classifier $h_0$}

\noindent\STATE \underline{\textsc{Encode:}} $S_0^Z\mathrm{=} \{z_{0i} \mathrm{=} \mathcal E (x_{0i})\}_{i=1}^m,S_T^Z \mathrm{=} \{z_{Tj} \mathrm{=} \mathcal E (x_{Tj})\}_{j=1}^n$

\STATE \underline{\textsc{Optimal Transport (OT):}} Solve for the OT plan $\gamma^{*}\in \mathbb{R}_{\geq 0}^{m \times n}$ between $S_0^Z$ and $S_T^Z$

\STATE \underline{\textsc{Cutoff:}} Use a cutoff threshold to keep $\cO(n\mathrm{+}m)$ elements of $\gamma^*$ above the threshold and zero out the rest \textrm{\texttt{\small //Only applies to the entropy-regularized version of OT}}

\STATE \underline{\textsc{Intermediate Domain Generation:}}

\FOR{$t = 1, \dots, T$}
    \STATE Initialize an empty set $S_t^Z$
    \FOR{each non-zero element $\gamma^*_{ij}$ of $\gamma^*$}
    
    \STATE $z \gets \frac{T-t}{T} z_{0i} + \frac{t}{T} z_{Tj}$ 
    \STATE Add $(z,\gamma^*_{ij})$ to $S_t$
    \ENDFOR
\ENDFOR

\STATE \underline{\textit{\textsc{Gradual Domain Adaptation:}}}
\FOR{$t = 1, \dots, T$}

    \STATE $h_t\mathrm{=}\mathrm{ST}(h_{t-1},S_t)$ \texttt{\small //Can also apply sample weights to losses based on $\gamma^*_{ij}$ }
\ENDFOR
\OUTPUT {Target-adapted classifier $h_T$}
\end{algorithmic}
\end{algorithm}

We present our proposed algorithm in Algorithm \ref{algo:main}. Notice that Algorithm \ref{algo:main} directly generates intermediate domains between the source and target domains. However, in practice, there might be a few given intermediate domains that can be used by GDA. In this case, one can simply treat each pair of consecutive domains as a source-target domain pair, and apply Algorithm \ref{algo:main} iteratively to the pairs of consecutive given domains from the source to target. 

Next, we explain the key designs of the proposed algorithm.

\subsubsection{Fast Computation of Optimal Transport (OT)} 
The super-cubic time complexity of solving the LP in~\eqref{eq:opt-plan} essentially prevents this optimal transport approach from being scaled up to large datasets. To remedy this issue, we propose to solve an approximate objective of the OT problem~\eqref{eq:opt-plan} when it takes too long to solve the original OT exactly. Specifically, we add an entropic regularization term to the objective~\eqref{eq:opt-plan}, turning it to be strictly convex, and the time complexity of solving this regularized objective is reduced to nearly $\cO(n^2)$ from the original $\cO(n^3 \log n)$~\citep{cuturi2013sinkhorn,dvurechensky2018computational}. However, the solution to this regularized objective, i.e., the OT plan $\gamma^*$, is not guaranteed to have at most $n+m-1$ elements anymore. Thus, the space complexity increases to $\cO(mn)$ from $\cO(m+n)$. In light of this challenge, we design a cutoff trick to zero out entries of tiny magnitude in $\gamma^*$ (see details in Algo.~\ref{algo:main}), reducing the space complexity back to $\cO(m+n)$. More details regarding this part are provided in Appendix~\ref{supp:algo}.

Note that beyond the Sinkhorn algorithm, several alternative approaches enhance the efficiency of OT computation. For instance, the Greenkhorn algorithm~\citep{altschuler2017near} improves the performance of the Sinkhorn algorithm, with a complexity of $\tilde{\cO}(n^2/\varepsilon^2)$, where $\varepsilon$ is the desired accuracy. Additionally, Low-rank Optimal Transport (LOT)~\citep{forrow2019statistical,scetbon2021low,scetbon2022linear} approaches the problem by reducing the size of measures before solving OT. This method specifically seeks couplings of low rank, which significantly reduces computational demands. Another approach, sliced-Wasserstein distance~\citep{bonneel2015sliced,kolouri2019generalized}, involves computing linear projection of high-dimensional distributions to one-dimensional distributions, then averaging the resulting Wasserstein distances, which can be computed using closed-form formulas. Given the varied applicability and use cases of these methods, we recommend that practitioners select the OT algorithm that best suits their specific needs.

\begin{figure}[t!]
    \centering
    \begin{subfigure}[t]{.495\linewidth}
    \includegraphics[width=0.85\linewidth]{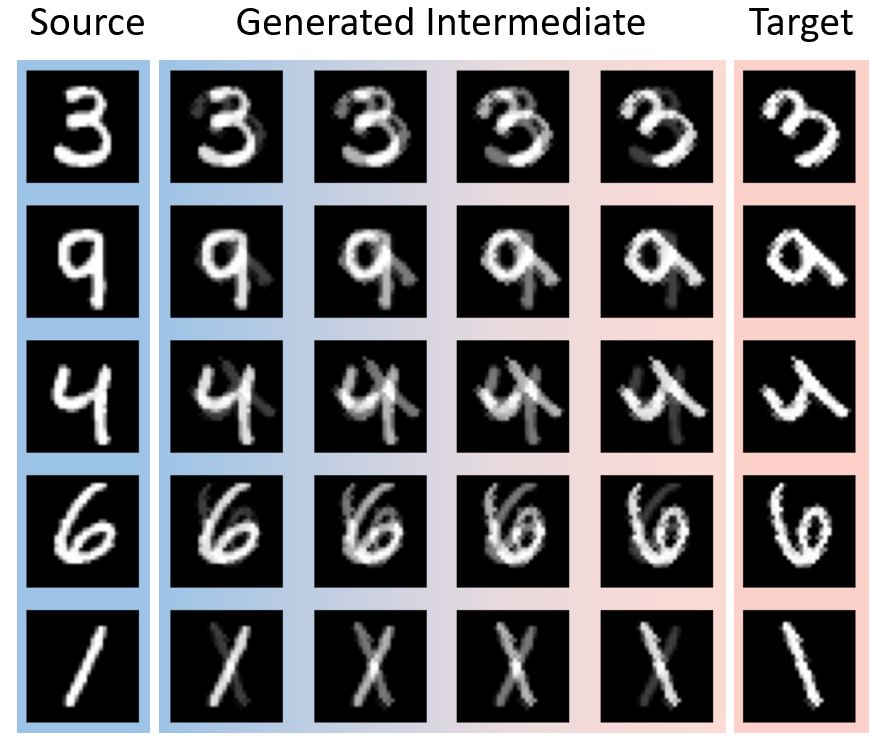}
    \caption{Input-space generation.}
    \label{fig:og_mnist}
  \end{subfigure}
\begin{subfigure}[t]{.495\linewidth}
    \centering\includegraphics[width=0.85\linewidth]{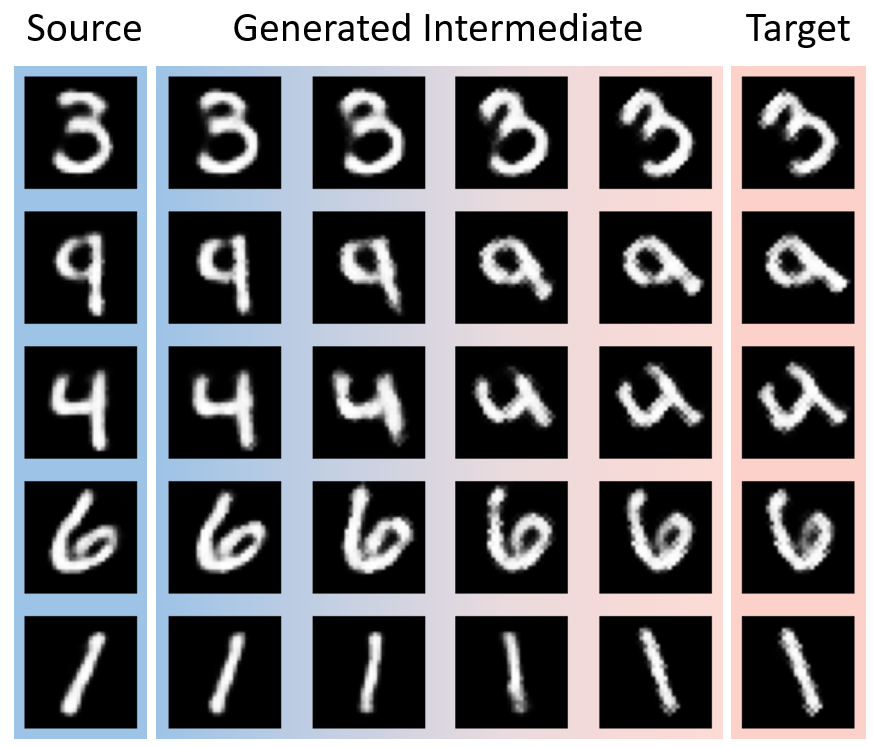}
    \caption{Feature-space generation.}
    \label{fig:latent_mnist}
  \end{subfigure}
    \caption{Samples from generated intermediate domains.}
    \label{fig:mnist}
\end{figure}

\begin{figure*}[t!]
    \centering
    \includegraphics[width=0.9\textwidth]{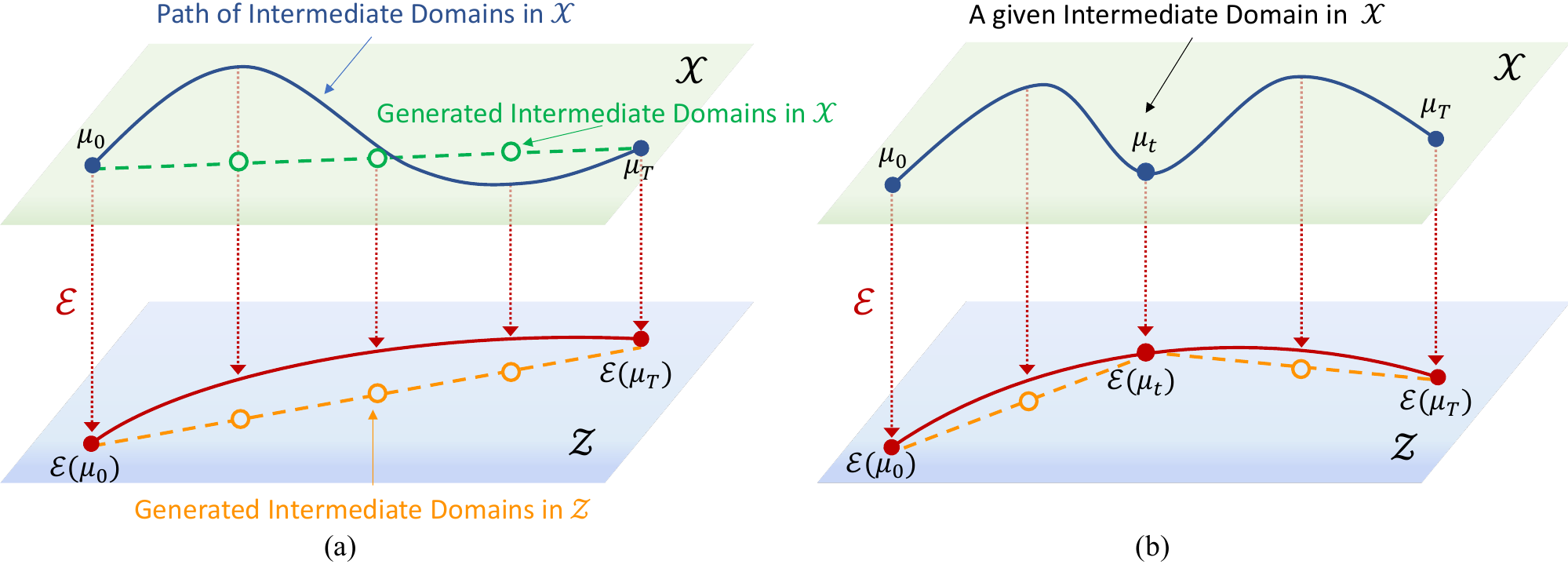}
    \caption{\small Illustration of the intermediate domain generation in GOAT. (a) without any given intermediate domain,
    (b) with one given intermediate domain.
    }
    \label{fig:manifold}
\end{figure*}

\subsubsection{Intermediate Domain Generation in Feature Space}
The intermediate domain generation approach described above directly generates data in the input space $\cX$. However, the generation does not have to be restricted to the input space. One can show that with a Lipschitz continuous encoder $\cE: \cX \mapsto \cZ$ mapping inputs to the feature space $\mathcal Z$ (i.e., $z \gets \cE(x)$ for any input $x$), the order of the generation bound \eqref{eq:gen-bound} stays the same\footnote{Some terms in the bound get multiplied by a factor of the Lipschitz constant of $\cE$.} (the proof is in Appendix \ref{supp:theory}). 

\paragraph{Feature Space vs. Input Space}
We use an encoder by default in Algorithm \ref{algo:main}, since we empirically observe that directly generating intermediate domains in the input space is usually sub-optimal (see~\Cref{fig:ablation:encoder} for a detailed analysis). To give the readers an intuitive understanding, we provide a demo of Rotated MNIST in~\Cref{fig:mnist}: if we apply the intermediate domain generation of~\Cref{algo:main} in the input space, the generated data do not approximate the digit rotation well; when applying the algorithm in the latent space of a VAE (fitted to the source and target data), the generated data (obtained by the decoder of the VAE) captures the digit rotation accurately. 

\Cref{fig:manifold}a explains this superiority of the feature space over the input space with a schematic diagram: the input-space Wasserstein geodesic can not well approximate the ground-truth distribution shift (e.g., rotation) due to the linearity of push-forward operators under the Euclidean metric; with a proper encoder $\cE$, the feature-space Wasserstein geodesic can capture the distribution shift more accurately.  

\paragraph{Leveraging the Given Intermediate Domain(s)}
With a given intermediate domain, we generate intermediate domains with GOAT between the two pairs of consecutive domains, respectively. \Cref{fig:manifold}b shows that this approach can make the generated domains closer to the ground-truth path of distribution shift, explaining why GOAT benefits from given intermediate domains.

\paragraph{Gradual Domain Adaptation (GDA) on Generated Intermediate Domains}
With the generated data of intermediate domains, one can run the GDA algorithm consecutively over the source-intermediate-target domains in the feature space. As for the choice of GDA algorithm, we adopt Gradual Self-Training (GST)~\citep{kumar2020understanding}, mainly due to its simplicity. Nevertheless, one can freely apply any other GDA algorithm on top of the generated domains.

\section{Experiments}\label{sec:exp}
Our goal of the experiment is to demonstrate the performance gain of training on generated intermediate domains in addition to given domains. We compare our method with gradual self-training~\citep{kumar2020understanding}, which only self-trains a model along the sequence of given domains iteratively. In Sec.~\ref{sec:exp:ablation}, we further analyze the choices of encoder $\mathcal E$ and transport plan $\gamma^*$ used by Algorithm \ref{algo:main}. More details of our experiments are provided in Appendix \ref{supp:exp}.

\subsection{Datasets}

\paragraph{Rotated MNIST}
A semi-synthetic dataset built on the MNIST dataset~\citep{mnist}, with 50K images as the source domain and the same 50K images rotated by 45 degrees as the target domain. Intermediate domains are evenly distributed between the source and target.

\paragraph{Color-Shift MNIST}
We normalize the pixel values of MNIST to be in $\left[0,1\right]$. We use 50K images as the source domain and the same 50K images with pixel values shifted by 1 as the target domain, i.e., the target pixel values are shifted to be in $\left[1,2\right]$. Intermediate domains are also evenly distributed.

\paragraph{Portraits~\citep{portraits}}
A real-world gender classification image dataset consisting of portraits of high school seniors from 1905 to 2013. Following~\citet{kumar2020understanding}, the dataset is sorted chronologically and split into a source domain (first 2000 images), 7 intermediate domains (next 14000 images), and a target domain (last 2000 images).

\paragraph{Cover Type~\citep{CoverType}}
A tabular dataset aiming at predicting the forest cover type at certain locations given 54 features. Following~\citet{kumar2020understanding}, we sort the data by the distance to water body in ascending order, splitting the data into a source domain (first 50K data), 10 intermediate domains (each with 40K data) and a target domain (final 50K data). 

\subsection{Implementation}

Our code is built in PyTorch~\citep{pytorch}, and our experiments are run on NVIDIA RTX A6000 GPUs. For Rotated MNIST, Color-Shift MNIST and Portraits, we use a convolutional neural network (CNN) of 4 convolutional layers of 32 channels followed by 3 fully-connected layers of 1024 hidden neurons, with ReLU activation. For Cover Type, we use a multi-layer perceptron (MLP) of 3 hidden layers with 256 hidden neurons. We also adopt common practices of Adam optimizer~\citep{adam}, Dropout~\citep{dropout}, and BatchNorm~\citep{batchnorm}.
To calculate the optimal transport plan between the source and target, we use the Earth Mover Distance solver from~\citep{pythonot}. 
The number of generated intermediate domains is a hyperparameter, and we show the performance for 1,2,3 or 4 generated domains between each pair of consecutive given domains. 
Following practices~\citep{kumar2020understanding}, in self-training, we filter out the 10\% data where the model's prediction is least confident at. 

When implementing of GOAT (\Cref{algo:main}), we take the first two conv layers as the encoder $\mathcal E$, and treat the layers after them as the classifier $h$. Sec. \ref{sec:exp:ablation} explains this choice.

\paragraph{Notes on number of generated domain} Although Eq. \eqref{equ:optt} shows the relationship between the optimal number of domains and source-target distance, it is still unclear what exact number should be chosen. To solve the problem, one can use a heuristic hyperparameter tuning approach. Specifically, a subset of the target set with highly confident pseudo-labels can be used as a validation set. Then, with all other components of the algorithm fixed, one can evaluate the performance using different numbers of domains on the target validation set and select the (empirically) optimal number of intermediate domains. However, as subsequent sections will demonstrate, the hyperparameter tuning stage is generally not necessary for the success of GOAT. Instead, our findings indicate that GOAT's performance is robust across varying numbers of domains. 

\subsection{Empirical Results}

\paragraph{Comparison with UDA methods}~We first empirically validate our claim that the traditional one-off UDA methods do not work well on datasets with large distribution shifts. Here, we compare GDA methods with three popular UDA methods: DANN~\citep{ganin2016domain}, DeepCoral~\citep{sun2016deep} and DeepJDOT~\citep{damodaran2018deepjdot}. These UDA methods do not have mechanisms to incorporate additional unlabeled data during training, so we use the source and target data as in the conventional UDA framework. In contrast, GDA methods such as GST and our GOAT have the capability to incorporate intermediate domains with unlabeled data. For illustration, we use two given intermediate domains for both GDA algorithms. It is important to note that under this setting, the amount of labeled data used in UDA and GDA is identical. We report the comparison in \Cref{Tab:uda}. The results demonstrate the advantage of the GDA methods over traditional UDA approaches, as GDA methods consistently outperform UDA methods with various types of distribution shifts. GOAT further improves the performance on top of GST, with a detailed discussion in subsequent paragraphs.

\begin{table}[t!]
\centering
\small
\caption{\small{GDA methods outperform one-off UDA methods on datasets with large distribution shifts.}}
\scalebox{0.9}{
    \begin{tabular}{ccccc}
    \toprule
     & Rotated MNIST & Color-Shift MNIST & Portraits & Cover Type \\
    \midrule
    DANN~\citep{ganin2016domain} & 44.6$\pm$2.3 & 56.5$\pm$3.2 & 73.8$\pm$1.5 & 63.3$\pm$1.6\\
    DeepCoral~\citep{sun2016deep} & 49.6$\pm$1.8 & 63.5$\pm$2.1 & 71.9$\pm$1.3 & 66.8$\pm$1.5 \\
    DeepJDOT~\citep{damodaran2018deepjdot} & 51.6$\pm$2.1 & 65.8$\pm$2.7 & 72.5$\pm$1.3 & 67.0$\pm$1.2 \\
    \midrule 
    GST (2 given domains) & 61.6$\pm$2.1 & 67.6$\pm$4.8 & 77.0$\pm$1.3 & 66.9$\pm$1.4 \\
    GOAT (2 given domains) & \textbf{70.3$\pm$ 2.4} & \textbf{90.3$\pm$1.4} & \textbf{79.9$\pm$1.2} & \textbf{69.8$\pm$1.4}\\
    \bottomrule    
    \label{Tab:uda}        
    \end{tabular}
}
\end{table}

\paragraph{Comparison with Gradual Self-Training}~We empirically compare our proposed GOAT with Gradual Self-Training (GST)~\citep{kumar2020understanding}. The results on Rotated MNIST, Color-Shift MNIST, Portraits and Cover Type are shown in \Cref{Tab:mnist,Tab:colormnist,Tab:portraits,Tab:covtype}. Each experiment is run 5 times with 95\% confidence interval reported. The leftmost column corresponds to the performance of GST only on given intermediate domains, which is equivalent to GOAT without any generated intermediate domain. 

\begin{table}[t!]
\begin{minipage}[t]{.5\textwidth}
\captionof{table}{\small Accuracy (\%) on Rotated MNIST.}
\vspace{-0.5em}
\resizebox{\columnwidth}{!}{
\setlength\tabcolsep{0.4em}
\begin{tabular}{c cccc c}
\toprule
{\footnotesize \# Given} & \multicolumn{5}{c}{\footnotesize \# Generated Domains of GOAT} \\
{\footnotesize Domains} & 0 (GST) & 1                             & 2                             & 3                  & 4                                   \\
\midrule
0 & \textbf{50.3$\pm$0.7}  & 48.5$\pm$2.2 & 47.2$\pm$1.7 & 48.2$\pm$2.7 & 47.5$\pm$2.8     \\
1 & 56.3$\pm$1.9 & 55.2$\pm$2.6 & 54.6$\pm$1.6 & \textbf{57.1$\pm$2.2} & 56.2$\pm$1.9 \\
2 & 61.6$\pm$2.1 & 68.0$\pm$1.4 & 67.0$\pm$2.2 & 68.1$\pm$2.2 & \textbf{70.3$\pm$2.4}\\
3 & 66.3$\pm$2.0 & 74.0$\pm$1.1 & \textbf{74.4$\pm$1.8} & 73.2$\pm$2.0 & 74.0$\pm$2.3 \\
4  & 75.5$\pm$2.0 & 83.8$\pm$2.0 & 84.0$\pm$1.6 & \textbf{86.4$\pm$2.0} & 82.7$\pm$1.8\\   
\bottomrule     
\label{Tab:mnist}           
\end{tabular}
}
\end{minipage}
\begin{minipage}[t]{.5\textwidth}
\captionof{table}{\small Accuracy (\%) on Color-Shift MNIST.}
\vspace{-0.5em}
\resizebox{\columnwidth}{!}{
\setlength\tabcolsep{0.4em}
\begin{tabular}{c cccc c}
\toprule
{\footnotesize \# Given} & \multicolumn{5}{c}{\footnotesize \# Generated Domains of GOAT} \\
{\footnotesize Domains} & 0 (GST) & 1                             & 2                             & 3                  & 4                                   \\
\midrule
0 & 40.5$\pm$5.5  & 54.4$\pm$6.9 & 63.2$\pm$4.1 & 75.7$\pm$3.8 & \textbf{79.1$\pm$3.0}     \\
1 & 54.2$\pm$5.9 & 74.7$\pm$5.3 & 79.5$\pm$2.9 & 79.3$\pm$3.4 & \textbf{85.3$\pm$3.8} \\
2 & 67.6$\pm$4.8 & 78.3$\pm$3.4 & 84.8$\pm$2.5 & 89.0$\pm$1.5 & \textbf{90.3$\pm$1.4}\\
3 & 73.9$\pm$7.6 & 80.9$\pm$6.9 & 87.4$\pm$4.2 & \textbf{90.7$\pm$2.3} & \textbf{90.4$\pm$1.5} \\
4  & 77.4$\pm$7.2 & 84.4$\pm$4.6 & \textbf{91.8$\pm$1.8} & 91.0$\pm$1.8 & \textbf{91.3$\pm$1.2}\\   
\bottomrule     
\label{Tab:colormnist}           
\end{tabular}
}
\end{minipage}
\begin{minipage}[t]{.5\textwidth}
\captionof{table}{\small Accuracy (\%) on Portraits.}
\vspace{-0.5em}
\resizebox{\columnwidth}{!}{
\setlength\tabcolsep{0.4em}
\begin{tabular}{c cccc c}
\toprule
{\footnotesize \# Given} & \multicolumn{5}{c}{\footnotesize \# Generated Domains of GOAT} \\
{\footnotesize Domains} & 0 (GST) & 1                             & 2                             & 3                  & 4                                   \\
\midrule
0 & 73.3$\pm$1.3  & \textbf{74.0$\pm$1.3}&73.5$\pm$2.2 & 73.6$\pm$2.5 & \textbf{74.2$\pm$2.5}     \\
1 & 74.5$\pm$1.6 & 76.4$\pm$1.3 & 75.5$\pm$2.6 & \textbf{76.8$\pm$1.5} & 74.7$\pm$1.7 \\
2 & 77.0$\pm$1.3 & 77.4$\pm$2.1 & 79.4$\pm$2.4 & \textbf{79.9$\pm$1.2} & 77.2$\pm$0.9       \\
3 & 80.7$\pm$2.3 & 80.9$\pm$1.6 & 81.8$\pm$1.3 & \textbf{82.3$\pm$1.3}&81.3$\pm$1.5 \\
4  & 82.0$\pm$1.4 & 82.8$\pm$1.5 & \textbf{83.6$\pm$1.5} & 82.4$\pm$1.4 & 81.8$\pm$1.6\\   
\bottomrule     
\label{Tab:portraits}           
\end{tabular}
}
\end{minipage}
\begin{minipage}[t]{.5\textwidth}
\captionof{table}{\small Accuracy (\%) on Cover Type.}
\vspace{-0.5em}
\resizebox{\columnwidth}{!}{
\setlength\tabcolsep{0.4em}
\begin{tabular}{c cccc c}
\toprule
{\footnotesize \# Given} & \multicolumn{5}{c}{\footnotesize \# Generated Domains of GOAT} \\
{\footnotesize Domains} & 0 (GST) & 1                             & 2                             & 3                  & 4                                   \\
\midrule
0 & 63.0$\pm$2.3 & 64.2$\pm$2.2 & 65.0$\pm$2.4 & \textbf{66.2$\pm$2.1} & \textbf{66.5$\pm$2.0}     \\
1 & 65.9$\pm$2.1 & 68.5$\pm$2.0 & 68.4$\pm$1.5 & \textbf{69.1$\pm$1.5} & \textbf{69.1$\pm$1.5}\\
2 & 66.9$\pm$1.4 & 68.9$\pm$1.6 & 68.4$\pm$2.1 & 69.3$\pm$1.1 & \textbf{69.8$\pm$1.4}\\
3 & 66.9$\pm$1.3 & 68.3$\pm$1.4 & \textbf{69.9$\pm$1.8} & 68.0$\pm$1.5 & 68.8$\pm$1.1 \\
4  & 67.7$\pm$1.7 & \textbf{69.6$\pm$2.1} & 68.1$\pm$2.0 & \textbf{69.7$\pm$1.2} & 69.4$\pm$2.0\\   
\bottomrule     
\label{Tab:covtype}           
\end{tabular}
}
\end{minipage}
\end{table}

In \Cref{Tab:mnist,Tab:colormnist,Tab:portraits,Tab:covtype}, the rows (“\# Given Domains”) indicate the number of \textit{given intermediate domains}. The columns (``\# Generated domains of GOAT'') represent the number of \textit{generated intermediate domains between \textbf{each pair} of consecutive given domains} (e.g., between the source domain and the first ground-truth intermediate domain, or between the $i$-th and $(i+1)$-th ground-truth intermediate domains). For instance, in the case of ``\# given domains = 3'' and ``\# generated domains = 3'', we have 5 ground-truth domains (source, target and 3 intermediate domains) and $4\times 3=12$ generated domains (since there are 4 pairs of adjacent domains along the sequence of 5 ground-truth domains), leading to 17 domains in total. 

\textit{Results:} i) From the columns of \Cref{Tab:mnist,Tab:colormnist,Tab:portraits,Tab:covtype}, we can observe that the performance of GOAT monotonically increases with more given intermediate domains, indicating that GOAT indeed benefits from given intermediate domains.
ii) From the rows of \Cref{Tab:mnist,Tab:colormnist,Tab:portraits,Tab:covtype}, we can see that with a fixed number of given domains, our GOAT can consistently outperform Gradual Self-Training (GST). The only exception is the case of Rotated MNIST without any given intermediate domain, which might be due to the challenge illustrated in Fig. \ref{fig:manifold}(a). Overall, the empirical results shown in these tables demonstrate that our GOAT can consistently improve gradual self-training (GST) with generated intermediate domains when only a few given intermediate domains are available.

\paragraph{Comparison on Intermediate Domain Generation }
In unsupervised domain adaptation (UDA), various algorithms have been developed to \textit{generate intermediate domains} to facilitate adaptation, such as~\citet{gong2019dlow,na2021fixbi,na2022covi}. Among these, CoVi~\citep{na2022covi} stands out for its exceptional (state-of-the-art) performance on UDA benchmarks. CoVi utilizes MixUp~\citep{zhang2018mixup} to generate synthetic data that are used to adapt models, and it also employs techniques of contrastive learning, entropy maximization and label consensus. In the GDA setting, it is applied in a similar manner as GST, where the adaptation is done sequentially on two adjacent domains, from the source to the target. 
To ensure a fair comparison, we fix the network structure and training recipe of CoVi to match the implementation of our GOAT, and present the results with $95\%$ confidence intervals over 5 random seeds. Since CoVi is a vision-specific model, we conduct the comparison on the three vision datasets: Rotated MNIST, Color-Shift MNIST and Portraits. The results are reported in \Cref{Tab:covi}. Our proposed algorithm, GOAT, demonstrates comparable or superior performance to CoVi across all numbers of given domains (0,1,2,3,4). This indicates that our algorithm is indeed powerful at i) generating high-quality intermediate domains useful for gradual domain adaptation and ii) utilizing given (ground-truth) intermediate domains. 

\begin{table}[t!]
\centering
\captionof{table}{\small Comparison with CoVi \citep{na2022covi} on vision datasets.}\label{Tab:covi}           
{\small
\scalebox{0.8}{
    \begin{tabular}{c|ccc|ccc|ccc}
    \toprule
    \# Given& \multicolumn{3}{c|}{Rotated MNIST} & \multicolumn{3}{c|}{Color-Shift MNIST} & \multicolumn{3}{c}{Portraits}\\ 
    {Domains} & GST & CoVi & GOAT & GST & CoVi & GOAT  & GST & CoVi & GOAT    \\
    \midrule
    0 & \textbf{50.3$\pm$0.7} & 48.4$\pm$2.1 & 48.5$\pm$2.2 & 40.5$\pm$5.5 & 40.0$\pm$5.2 & \textbf{79.1$\pm$3.0} & 73.3$\pm$1.3 & 73.7$\pm$3.5 & \textbf{74.2$\pm$2.5}     \\
    1 & 56.3$\pm$1.9 & \textbf{57.2$\pm$1.8} & \textbf{57.1$\pm$2.2} & 54.2$\pm$5.9 & 59.4$\pm$5.7 & \textbf{85.3$\pm$3.8} & 74.5$\pm$1.6 & 75.3$\pm$1.8 & \textbf{76.8$\pm$1.5} \\
    2 & 61.6$\pm$2.1 & 64.2$\pm$3.4 & \textbf{70.3$\pm$2.4} & 67.6$\pm$4.8 & 77.6$\pm$7.6 & \textbf{90.3$\pm$1.4} & 77.0$\pm$1.3 & \textbf{79.8$\pm$3.0} & \textbf{79.9$\pm$1.2}  \\
    3 & 66.3$\pm$2.0  & 71.4$\pm$1.9 & \textbf{74.4$\pm$1.8} & 73.9$\pm$7.6 & 86.4$\pm$4.7 & \textbf{90.4$\pm$1.5} & 80.7$\pm$2.3 & \textbf{82.3$\pm$1.4} & \textbf{82.3$\pm$1.3} \\
    4 & 75.5$\pm$2.0 & 80.7$\pm$3.4 & \textbf{86.4$\pm$2.0} & 77.4$\pm$7.2 & 90.9$\pm$4.0 & \textbf{91.3$\pm$1.2} & 82.0$\pm$1.4 & 83.1$\pm$1.9 & \textbf{83.6$\pm$1.5}\\   
    \bottomrule     
    \end{tabular}}
}
\end{table}

\begin{figure*}[t!]
    \centering
    \begin{subfigure}[t]{.49\linewidth}
    \includegraphics[width=0.85\linewidth]{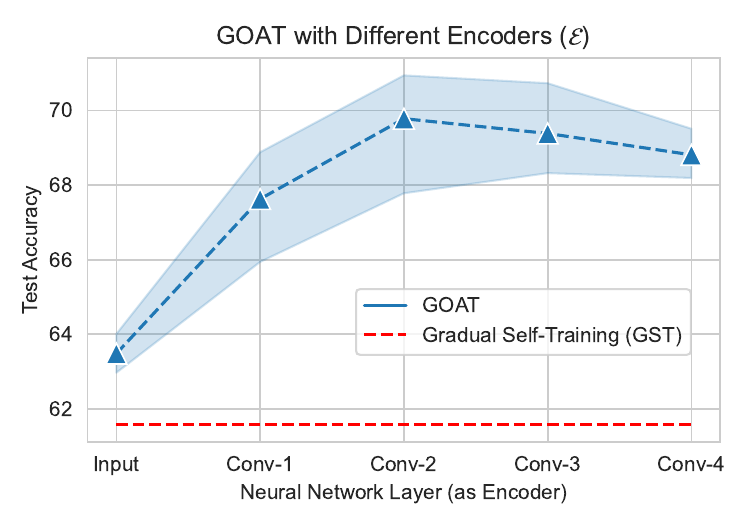}
    \caption{Choices of the encoder ($\mathcal E$).}
    \label{fig:ablation:encoder}
  \end{subfigure}
\begin{subfigure}[t]{.49\linewidth}
    \centering\includegraphics[width=0.85\linewidth]{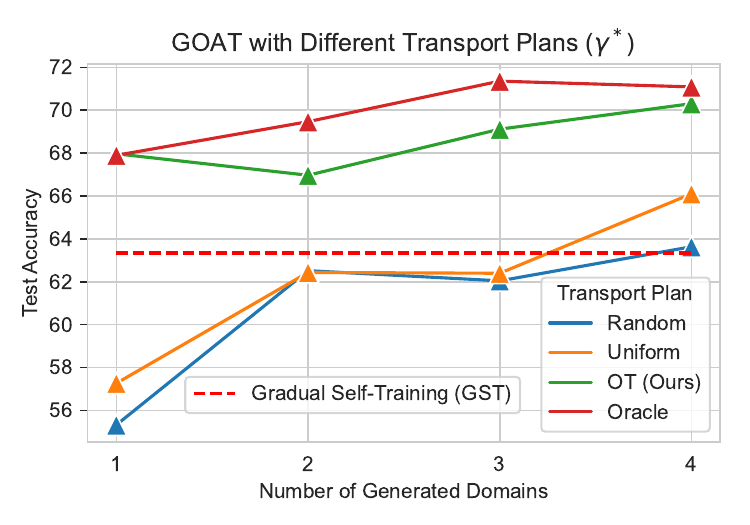}
    \caption{Choices of the transport plan ($\gamma^*$).}
    \label{fig:ablation:OT}
  \end{subfigure}
    \caption{\small Ablation studies on Rotated MNIST with 2 given intermediate domains. (a) \textit{Different neural net layers as encoders for the intermediate domain generation of GOAT.} One can see that input space is not suitable for intermediate domain generation, and the second convolutional layer (\textsc{Conv-2}) is optimal. (b) \textit{Different transport plans for the intermediate domain generation of GOAT.} Obviously, our optimal transport (OT) plan significantly outperforms the baseline transport plans (Random \& Uniform), and its performance is even close to the oracle.}
    \label{fig:ablation}
\end{figure*}

\subsection{Ablation Studies}\label{sec:exp:ablation}

\paragraph{Choice of Encoder ($\mathcal E$)}
Here, we study how the choice of the encoder (i.e., feature space) affects the performance of GOAT. Since we use a CNN, we can take each network layer as the feature space. Specifically, we consider the four convolutional layers and input space as candidate choices for the encoder. Once choosing a layer, we take all layers before it (including itself) as the encoder. In this ablation study, we use Rotated MNIST dataset with 2 given intermediate domains, and let GOAT generate 4 intermediate domains between consecutive given domains. From Fig. \ref{fig:ablation:encoder}, we can observe that directly applying GOAT in the input space performs significantly worse than the optimal choice, \textsc{Conv-2} (i.e., the second convolutional layer). This result justifies our use of an encoder for intermediate domain generation (instead of directly generating in the input space). Notably, Fig. \ref{fig:ablation:encoder} shows that deeper layers are not always better, showing a clear increase-then-decrease accuracy curve. Hence, we keep using \textsc{Conv-2} as the encoder for GOAT in all experiments.

\paragraph{Choice of Transport Plan ($\gamma^*$)}
In our Algorithm \ref{algo:main}, the data generated along the Wasserstein geodesic are essentially linear combinations of data from the pair of given domains, with weights (for each combination) assigned by the optimal transport (OT) plan $\gamma^*$. To validate that the performance gain of GOAT indeed comes from the Wasserstein geodesic estimation instead of just linear combinations, we conduct an ablation study on GOAT in Rotated MNIST with 2 given intermediate domains. Specifically, we consider four approaches to provide the transport plan $\gamma^*$: i) a random transport plan (weights are sampled from a uniform distribution), ii) a uniform transport plan (weights are the same for all combinations), iii) the optimal transport (OT) plan provided by Algorithm \ref{algo:main}, iv) the oracle transport plan\footnote{The target data of the Rotated MNIST dataset are obtained by rotating training data. Thus there is a one-to-one mapping between source and target data. The oracle plan is built from the one-to-one mapping, i.e., an element $\gamma^*_{ij}$ is non-zero if and only if $x_{0i}$ is rotated to $x_{Tj}$.}, which is the ground-truth transport plan in this study. For a fair comparison, when constructing the random and uniform plans, we ensure the number of non-zero elements is the same as that of the oracle plan (i.e., keeping the number of generated data the same). See more details in Appendix \ref{supp:exp}.

From Fig \ref{fig:ablation:OT}, we observe that, in general, the random and uniform plans do not obtain non-trivial performance gain compared with the baseline, the vanilla Gradual Self-Training (GST) without any generated domain. In contrast, our OT plan is significantly better and achieves similar performance as the oracle, demonstrating the high quality of the OT plan and justifying our algorithm design with the Wasserstein geodesic.
\section{Conclusion}
In this work, we study gradual domain adaptation. On the theoretical side, we provide a significantly improved analysis for the generalization error of the gradual self-training algorithm, under a more general setting with relaxed assumptions. In particular, compared with existing results, our bound provides an \emph{exponential} improvement on the dependency of the step size $T$, as well as a better sample complexity of $O(1/\sqrt{nT})$, as opposed to $O(1/\sqrt{n})$ as in the existing work. Based on the theoretical insight, we propose a novel algorithmic framework, Generative Gradual Domain Adaptation with Optimal Transport (GOAT), which automatically generates intermediate domains along the Wasserstein geodesic (between consecutive given domains) and applies GDA on the generated domains. Empirically, we show that GOAT can significantly outperform vanilla GDA when the given intermediate domains are scarce. Essentially, our GOAT is a promising framework that augments GDA with generated intermediate domains, leading GDA to be applicable to more real-world scenarios.

\bibliography{reference}

\newpage

\appendix

\section{Proof}\label{supp:proof}

\subsection{Proof of \cref{lemma:error-diff}}\label{supp:proof:error-diff}
\errordiff*

\begin{proof}
The population error difference of $h$ over the two domains (i.e., $\mu$ and $\nu$ is
\begin{align}
    |\eps_{\mu}(h) - \eps_{\nu}(h)| &= \left|\E_{x,y\sim \mu}[\ell(h(x), y)] - \E_{x',y'\sim\nu}[\ell(h(x'),y')]\right| \nonumber \\
    &= \left| \int \ell(h(x),y) d \mu - \int \ell(h(x'),y') d \nu  \right|\label{eq:err-diff}
\end{align}

Let $\gamma $ be an arbitrary coupling of $\mu$ and $\nu$, i.e., it is a joint distribution with marginals as $\mu$ and $\nu$. Then, \eqref{eq:err-diff} can be re-written and bounded as
\begin{align}
    |\eps_{\mu}(h) - \eps_{\nu}(h)| &= \left | \int \ell(h(x),y) d \mu - \int \ell(h(x'),y') d \gamma  \right|\\
    \text{(triangle inequality)}&\leq \int  \left |\ell(h(x),y) - \int \ell(h(x'),y') \right|d \gamma \\
    \text{($\ell$ is $\rho$-Lipschitz)} &\leq \int  \rho \left( \|h(x) - h(x')\| + \|y-y'\| \right) d\gamma \\
    \text{($h$ is $R$-Lipschitz)} &\leq \int  \rho  R\|x - x'\| +\rho \|y-y'\|  d\gamma \\
    \text{($R > 0$)} &\leq \int  \rho \sqrt{R^2 + 1} \left(\|x -x'\| + \|y-y'\|\right)  d\gamma 
\end{align}
Since $\gamma$ is an arbitrary coupling, we know that 
\begin{align}
    |\eps_{\mu}(h) - \eps_{\nu}(h)| &\leq \inf_{\gamma}  \int  \rho \sqrt{R^2 + 1} \left(\|x -x'\| + \|y-y'\|\right)  d\gamma \\
    &= \rho \sqrt{R^2 + 1} W_1(\mu, \nu)
\end{align}

Since the Wasserstein distance $W_p$ is monotonically increasing for $p \geq 1$, we have the following bound, 
\begin{align}\label{eq:error-diff-neighbour}
    |\eps_{\mu}(h) - \eps_{\nu}(h)|\leq \rho \sqrt{R^2 + 1} W_1(\mu, \nu) \leq \rho \sqrt{R^2 + 1} W_p(\mu, \nu)
\end{align}
\end{proof}

\subsection{Proof of \cref{prop:algorithm-stability}}\label{supp:proof:algo-stability}

\AlgoStability*

\begin{proof}
Define $\wh \eps_{\mu}(h) \coloneqq \frac{1}{|S|}\sum_{x\in S} \ell(h(x), y)$ as the empirical loss over the dataset $S$, where $S$ consists of samples i.i.d. drawn from $\mu(X)$ and $y$ is the ground truth label of $x$.

Then, we have the following sequence of inequalities:
\begin{align*}
    (\text{Use Lemma A.1 of~\citet{kumar2020understanding}})\quad \eps_{\mu}(h) &\leq \wh \eps_{\mu}( \hat h)+ \cO\left(R_n(\ell \circ \mathcal{H})+\sqrt{\frac{\log(1/\delta)}{n}}\right)\\
    \left(\text{since } h(x)=\hat h(x) ~\forall x \in S\right)~&= \widehat \eps_{\mu}(h)+ \cO\left(R_n(\ell \circ \mathcal{H})+\sqrt{\frac{\log(1/\delta)}{n}}\right)\\
    (\text{Use Lemma A.1 of~\citet{kumar2020understanding} again})~&\leq \eps_{\mu} (h) + \cO\left(2 R_n(\ell \circ \mathcal{H})+2 \sqrt{\frac{\log(1/\delta)}{n}}\right)\\
    (\text{By \cref{lemma:error-diff}})&\leq \eps_{\nu} (h) + \rho \sqrt{R^2+1} W_p(\mu,\nu) \\
    &\quad + \cO\left( R_n(\ell \circ \mathcal{H})+\sqrt{\frac{\log(1/\delta)}{n}}\right)\\
    (\text{By Talagrand's lemma with Assumption \ref{assum:Lipschitz-loss},\ref{assum:bounded-complexity}})&\leq \eps_{\nu} (h) + \rho \sqrt{R^2+1} W_p(\mu,\nu) \\
    &\quad + \cO\left( \frac{\rho B}{\sqrt n} +\sqrt{\frac{\log(1/\delta)}{n}}\right)\\
    &\leq \eps_{\nu} (h)  + \cO\left( W_p(\mu,\nu)+ \frac{\rho B}{\sqrt n} +\sqrt{\frac{\log(1/\delta)}{n}}\right)
\end{align*}

For the step using Talagrand's lemma~\citep{talagrand1995concentration}, the proof of Lemma A.1 of~\citet{kumar2020understanding} also involves an identical step, thus we do not replicate the specific details here. 
\end{proof}

\subsection{Proof of \cref{lemma:disc-bound}}\label{supp:proof:discrepancy-bound}
\DiscBound*

\begin{proof}
Within our setup of gradual self-training,
\begin{align*}
    \disc(\bq_{t}) &= \sup_{h\in \cH} \left( \eps_{t-1}(h) - \sum_{\tau=0}^{t-1} q_\tau \cdot \eps_{\tau}(h) \right) \\
    &= \sup_{h\in \cH} \left(  \sum_{\tau=0}^{t-1} q_\tau \left(\eps_{t-1}(h) -  \eps_{\tau}(h) \right)\right)\\
    &\leq \sup_{h\in \cH} \left(  \sum_{\tau=0}^{t-1} q_\tau |\eps_{t-1}(h) -  \eps_{\tau}(h) |\right)\\
    (\text{By \cref{lemma:error-diff}}) &\leq \rho \sqrt{R^2 + 1} \sum_{\tau=0}^{t-1}q_\tau\cdot (t-\tau-1)\Delta
\end{align*}
With $\bq_{t} = \bq_{t}^* = (\frac{1}{t},..., \frac{1}{t})$, this bound becomes
\begin{align*}
    \disc(\bq_t^*) \leq \rho \sqrt{R^2 + 1} \sum_{\tau=0}^{t-1}q_\tau\cdot (t-\tau-1)\Delta = \rho \sqrt{R^2 + 1}~ \frac{t}{2} \Delta = \cO(t\Delta)
\end{align*}
and it is trivial to show that this upper bound is smaller than any other $\bq_t$ with $\bq_t \neq \bq_t^*$.
\end{proof}

\subsection{Proof of \cref{thm:gen-bound}}\label{supp:proof:gen-bound}

\begin{figure}[t!]
\begin{center}
\centerline{\includegraphics[width=.9\textwidth]{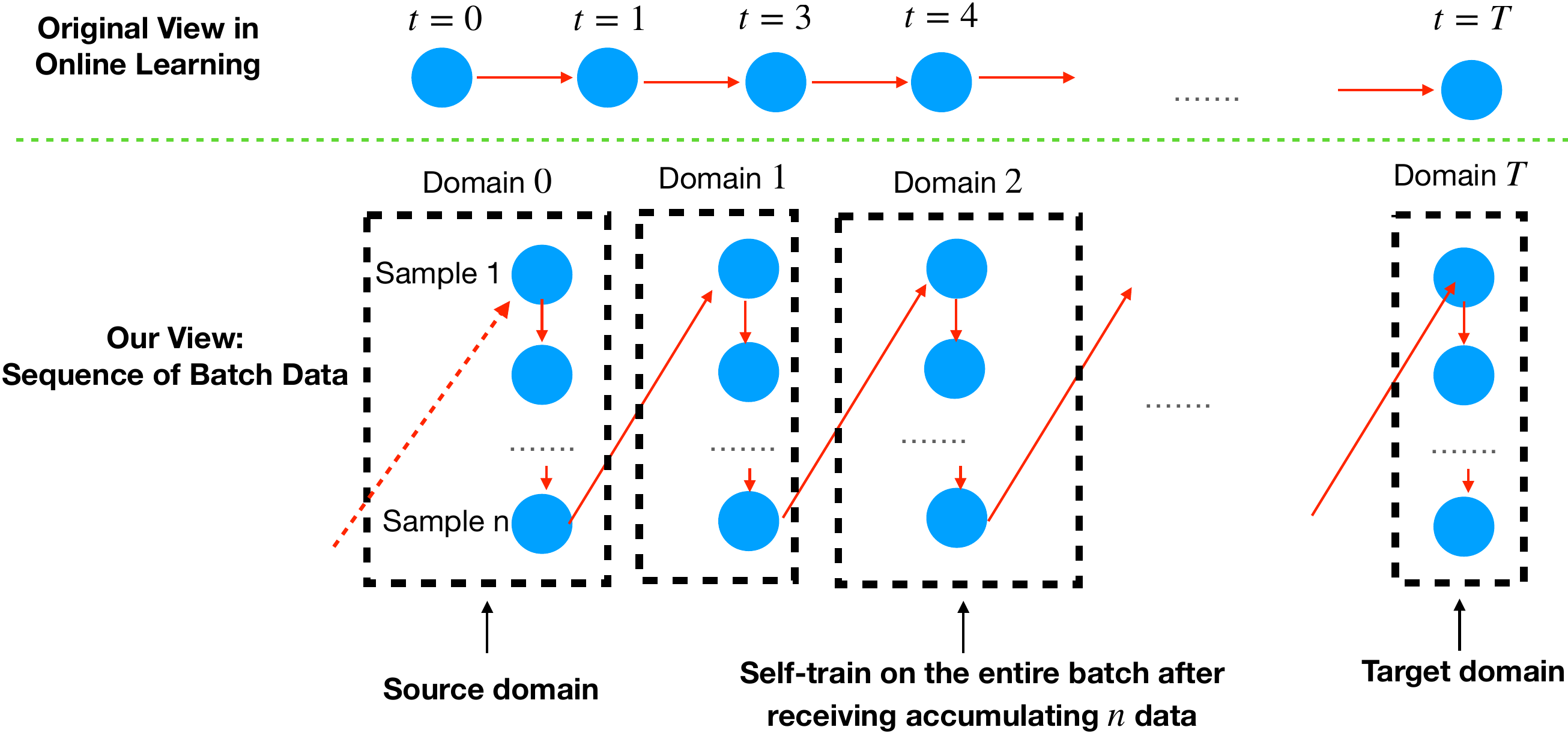}}
\caption{Our reductive view of gradual self-training that is helpful to \cref{thm:gen-bound}.}
\label{fig:our-view}
\end{center}
\vskip -.4in
\end{figure}

\GenBound*

\paragraph{A Reductive View of the Learning Process of Gradual Self-Training}

If we directly apply Corollary 2 of~\citet{kuznetsov2020discrepancy}, we can obtain a generalization bound as
\begin{align}
        \eps_{\mu_T} (h) &\leq \sum_{t=0}^{T} q_t \eps_{\mu_t} (h) + \disc(\bq_{T+1}) + \|\bq_{T+1}\|_2 + 6M \sqrt{4\pi \log T} \mathcal R_T^{\mathrm{seq}} (\ell \circ \mathcal H) \nonumber\\
        &\quad + M \|\bq_{T+1}\|_2 \sqrt{8 \log \frac{1}{\delta }}\nonumber\\
        &\leq \sum_{t=0}^{T} q_t \eps_{\mu_t} (h) + O(T \Delta) + \cO(\frac {1} {\sqrt{T}}) + 6M \sqrt{4\pi \log T} \mathcal R_T^{\mathrm{seq}} (\ell \circ \mathcal H) +  \cO(\sqrt{\frac{\log \frac{1}{\delta }}{T}})\label{eq:supp:proof:gen-bound:naive-bound}
\end{align}
where $M$ is an upper bound on the loss (\cref{lemma:bounded-loss} proves such a $M$ exists), and the last inequality is obtained by setting $\bq_{T+1} = \bq^*_{T+1} = (\frac{1}{T+1},\dots, \frac{1}{T+1})$.

A typical generalization bound involves terms with dependence on $N$ (the training set size), usually in the form $\cO(\sqrt{\frac 1 N})$, and these terms vanish in the infinite-sample limit (i.e., $N\rightarrow \infty$). These terms also appear in standard generalization bounds of unsupervised domain adaptation~\citep{ben2007analysis,zhao2019domain}, where $N$ becomes the number of available unlabelled data in the target domain.

In the case of gradual domain adaptation, the total number of available unlabelled is $Tn$, and we would expect $Tn$ will appear in a form similar to $\cO(\sqrt{\frac{1}{nT}})$, which vanishes in the infinite-sample limit (i.e., $nT\rightarrow \infty$). However, the generalization bound \eqref{thm:gen-bound} has terms $\cO(\sqrt{\frac 1 T})$ and $\cO(\sqrt{\frac{\log \frac 1 \delta}{T}})$, which does not vanish even with infinite data per domain, i.e., $n\rightarrow \infty$ (certainly results in $Tn\rightarrow \infty$).

We attribute this issue to the coarse-grained nature of online learning analyses such as~\citet{kuznetsov2016time,kuznetsov2020discrepancy}, which do not take data size per domain into consideration.

To address this issue, we propose a novel reductive view of the entire learning process of gradual self-training, leading to a more fined-grained generalization bound than Eq. \eqref{eq:supp:proof:gen-bound:naive-bound}. 

We draw a diagram to illustrate this reductive view in Fig. \ref{fig:our-view}. Specifically, instead of viewing each domain as the smallest element, we zoom in to the sample-level and view each sample as the smallest element of the learning process. We view the gradual self-training algorithm as follows: it has a fixed data buffer of size $n$, and each newly observed sample is pushed to the buffer; the model updates itself by self-training once the buffer is full; after the update, the buffer is emptied. Notice that this view does not alter the learning process of gradual self-training.

With this reductive view, the learning process of gradual self-training consists of $nT$ smallest elements (i.e., each sample is a smallest element), instead of $T$ elements (i.e., each domain is a smallest element) in the view of online learning works~\citep{kuznetsov2016time,kuznetsov2020discrepancy}. As a result, terms of order $\cO\left(\sqrt{\frac 1 T}\right)$ in \eqref{eq:supp:proof:gen-bound:naive-bound} becomes $\cO\left(\sqrt{\frac{1}{nT}}\right)$, and terms of order $\cO\left(\frac{T}{n}\right)$ also vanish as $n\rightarrow \infty$. Notably, the upper bounds on the terms $\sum_{t=0}^{T} q_t \eps_{\mu_t} (h)$ and $\disc(\bq_{T+1}) $ in \eqref{eq:naive-gen-bound} do not become larger with this view, since there is no distribution shift within each domain (e.g., the learning process over the first $n$ samples in Fig. \ref{fig:our-view} does not involve any distribution shift, and the iteration $n-1\mapsto n$ incurs a distribution shifts, since the $(n-1)$-th sample is in the first domain while the $n$-th sample is in the second domain).

With this reductive view, we can finally obtain a tighter generalization bound for gradual self-training without the issues mentioned previously.
\begin{proof}
With the inductive view introduced above, we can improve the naive bound \eqref{eq:supp:proof:gen-bound:naive-bound} to
\begin{align*}
 \eps_{\mu_T} (h_{T}) &\leq  \sum_{t=0}^{T} \sum_{i=0}^{n-1} q_{nt+i} \eps_{\mu_t} (h_{T}) + \disc(\bq_{n(T+1)}) + \|\bq_{n(T+1)}\|_2 + 6M \sqrt{4\pi \log nT} \mathcal R_{nT}^{\mathrm{seq}} (\ell \circ \mathcal H)\eq \label{eq:supp:proof:gen-bound:first-line}\\
 &\qquad + M \|\bq_{n(T+1)}\|_2 \sqrt{8 \log \frac{1}{\delta }}\nonumber\\    
 &\leq  \frac {1} {T+1} \sum_{t=0}^{T} \eps_{\mu_t}(h_{T}) +  \rho \sqrt{R^2 + 1}~ \frac{T+1}{2} \Delta  + \frac{1}{\sqrt{nT}}+ 6M \sqrt{4\pi \log nT}R_{nT}^{\mathrm{seq}} (\ell \circ \mathcal H) \nonumber\\
 &\quad + M \sqrt{\frac{8\log 1 / \delta}{ nT}} \\
 &\leq \eps_{\mu_0}(h_0) +  \cO\left(T\Delta+ T\sqrt{\frac{\log 1 / \delta}{n}}+ \frac{1}{\sqrt{nT}} + \rho R\sqrt{\frac{(\log nT)^7}{nT}} + \sqrt{\frac{\log 1/\delta}{nT}}\right )
 \end{align*}
 where $\bq_{n(T+1)}$ is taken as $\bq_{n(T+1)}= \bq_{n(T+1)}^*=(\frac{1}{n(T+1)}), \dots, \frac{1}{n(T+1)})$. We used the following facts when deriving the inequalities above:
 \begin{itemize}
     \item The first term of \eqref{eq:supp:proof:gen-bound:first-line} has the following bound
     \begin{align}\sum_{t=0}^{T} \sum_{i=0}^{n-1} q_{nt+i} \eps_{\mu_t} (h_{T}) &= \frac{1}{T+1} \sum_{t=0}^{T} \eps_{\mu_t}(h_{T})\nonumber\\
     &\leq \eps_{\mu_0}(h_0) +  \cO(T\Delta) + \cO\left(\frac{1}{\sqrt n}+ T\sqrt{\frac{\log 1 / \delta}{n}}\right)
     \end{align}
     which is obtained by recursively apply \cref{lemma:error-diff} and \cref{prop:algorithm-stability} to each term in the summation. For example, the last term in $\sum_{t=0}^{T} \eps_{\mu_t}(h_{T})$ can bounded by \cref{prop:algorithm-stability} as follows
     \begin{align*}
    (\text{By \cref{prop:algorithm-stability}})\quad  \eps_{T}(h_{T}) &\leq \eps_{\mu_{T-1}} (h_{T-1}) + \cO\left(W_p(\mu_T,\mu_{T-1}) + \frac{1}{\sqrt n}+\sqrt{\frac{\log 1 / \delta}{n}}\right)\\
    (\text{Same as the above step})&\leq \ldots \\
    &\leq \eps_{\mu_{0}} (h_{0}) + \cO(T\Delta+ \cO\left(T\sqrt{\frac{\log 1 / \delta}{n}}\right)\eq\label{eq:supp:proof:gen-bound:1st-term:last-term}
    \end{align*}
    and the second last term can be bounded similarly with the additional help of \cref{lemma:error-diff}
    \begin{align*}
         (\text{By \cref{lemma:error-diff}})\quad  \eps_{T-1}(h_{T}) &\leq \eps_{\mu_{T}} (h_{T}) + \cO(W_p(\mu_{T},\mu_{T-1})) \\
     (\text{Apply Eq. \eqref{eq:supp:proof:gen-bound:1st-term:last-term}})&\leq \eps_{\mu_{0}} (h_{0}) + T \Delta + \cO\left( T\sqrt{\frac{\log 1 / \delta}{n}}\right)
    \end{align*}
    All the rest terms (i.e., $\eps_{T-2}(h_{T}),\dots,\eps_{0}(h_{T})$) can be bounded in the same way.
\item The second term of \eqref{eq:supp:proof:gen-bound:first-line} can be bounded by applying \cref{lemma:disc-bound}.
\item The value of $R_{nT}^{\mathrm{seq}} (\ell \circ \mathcal H)$ can be bounded by combining \cref{lemma:seq-rademachor-composite} and Example \ref{example:neural-net}.
 \end{itemize}
\end{proof}

\subsection{Helper Lemmas}\label{supp:proof:helper-lemmas}

\begin{restatable}[Bounded Loss]{lem}{BoundedLoss}\label{lemma:bounded-loss}
For any $x\in \mathcal X, y \in \mathcal Y, h\in \mathcal H$, the loss $\ell(x,y)$ is upper bounded by some constant $M$, i.e., $l(h(x), y)\leq M$.
\end{restatable}
\begin{proof}
Notice that i) the input $x$ is bounded in a compact space, specifically, $\|x\|_2 \leq 1$ (ensured by Assumption \ref{assum:input-bound}), ii) $y$ lives in a compact space in $\bR$ (defined in Sec. \ref{sec:setup}), iii) the hypothesis $h\in \mathcal H$ is $R$-Lipschitz, and iv) the loss function $\ell$ is $\rho$-Lipschitz.

Combining these conditions, one can easily find that there exists a constant $M$ such that $l(h(x), y)$ for any $x\in \mathcal X, y\in \mathcal Y, h\in \mathcal H$.
\end{proof}

\begin{restatable}[Lemma 14.8 of~\citet{rakhlin2014notes}]{lem}{SeqRademacherComposite}\label{lemma:seq-rademachor-composite} For $\rho$-Lipschitz loss function $l$, the sequential Rademacher complexity of the loss class $\ell \circ \mathcal H$ is bounded as
\begin{align}
     \mathcal{R}^{\mathrm{seq}}_T(\ell \circ \mathcal H)  \leq  \cO(\rho \sqrt{(\log T)^3}) \mathcal{R}^{\mathrm{seq}}_T(\mathcal H)
\end{align}
\end{restatable}

\begin{proof}
See~\citet{rakhlin2014notes}.
\end{proof}

\subsection{Derivation of the Optimal $T$}\label{supp:proof:optimal-T}

In Sec. \ref{sec:optimal-path}, we show a variant of the generalization bound in \eqref{eq:simplified} as
\begin{equation}
\eps_{T}(h_T) \leq  \eps_{0}(h_0) \mathrm{+} \inf_{\mathcal{P}} \widetilde{\cO}\biggl(T\deltam \mathrm{+}\frac{T}{\sqrt n} \mathrm{+} \sqrt{\frac{1}{nT}}~\biggr)
\end{equation}
where $\deltam$ is an upper bound on the average $W_p$ distance between any pair of consecutive domains along the path, i.e., $\Delta_{\max} \geq \frac{1}{T}\sum_{t=1}^T W_p(\mu_{t-1}, \mu_t)$.

Given that $T,\deltam,n$ are all positive, we know there exists an optimal $T=T^*$ that minimizes the function
\begin{align}
    f(T)\coloneqq T\deltam \mathrm{+}\frac{T}{\sqrt n} \mathrm{+} \sqrt{\frac{1}{nT}}~ ,
\end{align}
and one can straightforwardly derive that
\begin{align}
    T^* = \left(\frac{1}{2(1+\deltam \sqrt{n}~)}\right)^{\frac 2 3}~.
\end{align}
\begin{proof}
The derivative of $f(T)$ is
\begin{align}
f'(T) = \deltam + \frac{1}{\sqrt n} - \frac{1}{2\sqrt{n}}T^{-\frac 3 2}~,
\end{align}
and the second-order derivative of $f(T)$ is
\begin{align}\label{eq:supp:optimal-T:2nd-grad}
f''(T) =   \frac{3}{4\sqrt{n}}T^{-\frac 5 2}~.
\end{align}
Eq. \eqref{eq:supp:optimal-T:2nd-grad} indicates that $f(T)$ is strictly convex in $T\in(0,\infty)$. Then, we only need to solve for the equation
\begin{align}
    f'(T) = 0 
\end{align}
as $T\in (0,\infty)$, which gives our the solution
\begin{align}
    T^* = \left(\frac{1}{2(1+\deltam \sqrt n~)}\right)^{\frac 2 3} ~.
\end{align}
\end{proof}





\section{Theoretical Arguments}\label{supp:theory}

\paragraph{On \Cref{prop:domain-path}} The inequality in \eqref{eq:prop:domain-path} holds true since the Wasserstein distance metric $W_p$ is known to enjoy the property of triangle inequality. In \eqref{eq:prop:domain-path}, the equality is obtained as the intermediate domains $\mu_1,\dots,\mu_{T-1}$ sequentially fall along the Wasserstein geodesic between $\mu_0$ and $\mu_T$, since the geodesic is defined as the shortest path of distributions connecting $\mu_0$ and $\mu_T$ under the $W_p$ metric. 

\paragraph{On \Cref{prop:LP}} This linear program (LP) formulation of optimal transport is also called Kantorovich LP in the literature. One can find details and proof of Kantorovich LP in~\citep{peyre2019computational}.

\paragraph{On the Encoder} With a $\rho_{\mathcal E}$-Lipschitz continuous encoder $\cE: \cX \mapsto \cZ$ mapping inputs to the feature space $\mathcal Z$ (i.e., $z \gets \cE(x)$ for any input $x$), the order of the generation bound \eqref{thm:gen-bound} stays the same. The reason is as follows: The bound \eqref{thm:gen-bound} is linear in terms of $\rho_h$ \footnote{The dependence on $\rho_{h}$ is hidden with the big-O notation in \eqref{thm:gen-bound}}, which is the Liphschitz constant of the classifier $h$; With the encoder $\mathcal E$, one can effectively view the whole encoder-classifier model as $f : \mathcal X \mapsto \mathcal Y$ such that $f(x) = h(\mathcal E(x))$; Then, the Liphschitz constant of $f$ is obviously $\rho = \rho_{\mathcal E} \rho_{h}$ since $f$ is a composite function of $h \circ \mathcal E$; Finally, replacing $h$ with $f$ in the analysis, one can see that the order of the bound \eqref{thm:gen-bound} stays the same, with some terms getting multiplied by a factor of $\rho_{\mathcal E}$ (i.e., equivalent to replacing the term $\rho_h$ with $\rho=\rho_{\mathcal E}\rho_h$ in the bound).

\section{More Details on the Proposed Algorithm}\label{supp:algo}
To reduce the $\cO(n^3 \log n)$ complexity of the exact OT calculation to $\cO(n^2)$, we can solve the entropy-regularized OT problem~\cite{cuturi2013sinkhorn} instead. Consider source data $\{x_{0i}\}_{i=1}^{m}$ and target data $\{x_{Tj}\}_{i=1}^{n}$, the entropy-regularized OT plan $\gamma^*_\lambda$ under the transport cost function $c$ is obtained by solving
\begin{align}\label{eq:sinkhorn}
\begin{split}
\gamma^*_\lambda=\argmin _{\gamma \in \mathbb{R}_{\geq 0}^{m \times n}} \sum_{i, j} \gamma_{i, j} c(x_{0i},x_{Tj})+\lambda \sum_{i,j}\gamma_{i,j}\log\gamma_{i,j},\\
\text { s.t. }~ \gamma \bm{1}_{n}= \frac{1}{m} \bm{1}_{m}~\text{ and }~ \gamma^{T} \bm{1}_m = \frac{1}{n}\bm{1}_n,
\end{split}
\end{align}
where $\lambda$ is a regularization coefficient. The low computational complexity comes at the cost of a dense optimal transport plan, i.e., $\gamma_\lambda^*$ is generally a dense matrix rather than a sparse one\footnote{As we discussed in Sec. \ref{sec:algo:analysis}, $\gamma^*$ has at most $n+m-1$ non-zero entries, thus it is a sparse matrix.}. Thus, $\cO(mn)$ non-zero entries will be generated in $\gamma_\lambda^*$, and this quadratic space complexity becomes intractable for large datasets. To remedy this issue, we design two methods to zero out insignificant entries in $\gamma^*_\lambda$ to reduce the space complexity:
\begin{enumerate}
    \item \textbf{Small-value cutoff.} Although the transport plan $\gamma_\lambda^*$ resulted from entropy-regularized OT is dense, most entries still have values close to 0. Those entries of tiny magnitude can be zeroed out without having a noticeable impact on the final results.
    \item \textbf{Confidence cutoff.} Consider the one-hot encoded matrix of source labels $Y_0\in\{0,1\}^{m\times \text{\#class}}$ and the entropy-regularized OT plan $\gamma^*_\lambda$. The logits of target prediction by optimal label transport is
    \begin{align}
        \widehat{Y}_T={\gamma^*_\lambda}^\trans Y_0.
    \end{align}
    Then, we can calculate a confidence score for each target prediction by the logits. Using a certain confidence threshold, the target samples that the transport plan is unconfident with can be filtered out, making the transport plan more sparse.
\end{enumerate}
With proper choices of cutoff values, those methods can reduce the space complexity from $\cO(mn)$ to $\cO(m+n)$ without noticeable compromise on the final performance.

\subsection{Number of Intermediate Domains}
The number of intermediate domains can be considered as a hyperparameter. The theory shows that there exists an optimal number $T^*$ in terms of self-training performance in the GDA setting:
\begin{align} \label{eq:num-domains}
    T^*=\max\left\{\frac{L}{\Delta}, \Tilde{\cO}\left(\left( \frac{1}{1+\Delta\sqrt{n}} \right)^{2/3} \right)  \right\},
\end{align}
where $L$ is the $W_p$ distance between the source and target and $\Delta$ is the average $W_p$ distance between any pair of consecutive domains. 

Although Eq. \eqref{eq:num-domains} shows the relationship between the optimal number of domains and source-target distance, it is still unclear what exact number should be chosen. To solve the problem, we use a heuristic hyperparameter tuning approach. Specifically, we use a subset of the target set with highly confident pseudo-labels as a validation set. Then, with all other components of the algorithm fixed, we evaluate the performance using different numbers of domains on the target validation set and select the (empirically) optimal number of intermediate domains.

\section{More Details on Experiments}
\label{supp:exp}

\paragraph{Network Implementation.} 
For the 4-layer CNN encoder used in experiments on Rotated MNIST and Portraits, we use convolutional layers with kernel size 3 and SAME padding. During self-training, we train on each domain for 10 epochs. Empirically, we verify that regularization techniques are important for the success of gradual self-training, including using dropout layers and early stopping.

For the VAE used to produce Fig. \ref{fig:mnist}, we use 4 convolutional layers with kernel size 3 and max-pooling, followed by a fully-connected layer with 128 neurons as the encoder. For the decoder, we use four deconvolutional layers with kernel size 3~\citep{kingma2014VAE}. We use ReLU activation for the layers. The encoder and decoder are jointly trained on data from source and target in an unsupervised manner with the Adam optimizer~\citep{adam} (learning rate as $10^{-4}$ and batch size as 512).

\paragraph{Encoder Pretraining.} 
We pretrain an encoder on the given domains. During pretraining, we use a 3-layer MLP on top of the encoder and perform self-training on the given domains. Specifically, we first fit the model on the source domain, then iteratively use the model to pseudo-label the next domain and self-train on it. After pretraining, the MLP is discarded and the encoder is fixed to provide features for the downstream tasks.

\paragraph{OT ablation.} 
When designing different plans, we make sure that the number of non-zero entries is equal so that in the domains generated by those plans, the amount of data is the same. For the random plan, we first initialize a zero matrix, then sample the same amount of entries as the ground-truth plan in the matrix, and fill in a weight value between 0 to 1 uniformly at random. For the uniform plan, we use the same procedure except that we fill in the same weight for each sampled entry. In the end, we normalize the matrix.

\end{document}